\documentclass[twoside,11pt]{article}

\usepackage[preprint]{jmlr2e}

\usepackage{caption}
\usepackage{algorithm}
\makeatletter
\renewcommand{\fnum@algorithm}{\fname@algorithm}
\makeatother
\usepackage{scalerel}
\usepackage{mathrsfs}
\usepackage{enumitem}
\usepackage{bbm}
\usepackage{tikz}
\usepackage{float}
\usetikzlibrary{backgrounds}
\usepackage{color}
\usepackage{graphicx}
\usepackage{latexsym}
\usepackage{amsfonts}
\usepackage{pifont,xspace,epsfig, wrapfig}
\usepackage{amsmath, amssymb} %
\usepackage{multirow}
\usepackage{dsfont}
\usepackage{array}
\usepackage{adjustbox}

\usepackage{bookmark}

\newcommand{\smallminus}{\scalebox{0.7}[1.0]{$-$}}
\newcommand{\smallequality}{\scalebox{0.75}[1.0]{$=$}}

\DeclareSymbolFont{AMSb}{U}{msb}{m}{n}
\DeclareMathSymbol{\N}{\mathbin}{AMSb}{"4E}
\DeclareMathSymbol{\Z}{\mathbin}{AMSb}{"5A}
\DeclareMathSymbol{\R}{\mathbin}{AMSb}{"52}
\DeclareMathSymbol{\Q}{\mathbin}{AMSb}{"51}
\DeclareMathSymbol{\erert}{\mathbin}{AMSb}{"50}
\DeclareMathSymbol{\I}{\mathbin}{AMSb}{"49}
\DeclareMathSymbol{\C}{\mathbin}{AMSb}{"43}

\def\argmax{\mbox{\rm argmax}}
\def\argmin{\mbox{\rm argmin}}
\newcommand{\mynote}[2]{{\textcolor{#1}{ #2}}}

\definecolor{gray}{gray}{0.4}
\newcommand{\gray}[1]{\mynote{gray}{{\footnotesize #1}}}

\newcommand{\remove}[1]{}

\newtheorem{mytheorem}{Theorem}[section]
\newtheorem{mylemma}[mytheorem]{Lemma}
\newtheorem{mydefinition}[mytheorem]{Definition}
\newtheorem{myremark}[mytheorem]{Remark}
\newtheorem{notation}[mytheorem]{Notation}

\newtheorem{claim}[mytheorem]{Claim}

\newtheorem{observation}[mytheorem]{Observation}

\newcommand{\AAA}{\mathcal A}
\newcommand{\BBB}{\mathcal B}

\newcommand{\HHH}{\mathcal H}

\newcommand{\III}{\mathcal I}

\newcommand{\eps}{\varepsilon}

\newcommand{\cost}{{\rm cost}}
\newcommand{\diam}{{\rm diam}}

\newcommand{\GAPTR}{\operatorname{\rm GAP-TR}}
\def\view{{\rm{View}}}

\newcommand{\poly}{\mathop{\rm poly}}

\def\E{\operatorname*{\mathbb{E}}}

\def\Q{\operatorname*{\mathbb{Q}}}
\def\poly{\mathop{\rm{poly}}\nolimits}

\def\OPT{\mathop{\rm{OPT}}\nolimits}
\def\opt{\mathop{\scriptscriptstyle \rm opt}}

\makeatletter
\newcommand{\thickhline}{%
    \noalign {\ifnum 0=`}\fi \hrule height 1pt
    \futurelet \reserved@a \@xhline
}
\newcolumntype{"}{@{\hskip\tabcolsep\vrule width 1pt\hskip\tabcolsep}}
\makeatother

\makeatletter
\newlength{\fboxhsep}
\newlength{\fboxvsep}
\newlength{\fboxtoprule}
\newlength{\fboxbottomrule}
\newlength{\fboxleftrule}
\newlength{\fboxrightrule}
\def\@frameb@xother#1{%
  \@tempdima\fboxtoprule
  \advance\@tempdima\fboxvsep
  \advance\@tempdima\dp\@tempboxa
  \hbox{%
    \lower\@tempdima\hbox{%
      \vbox{%
        \hrule\@height\fboxtoprule
        \hbox{%
          \vrule\@width\fboxleftrule
          #1%
          \vbox{%
            \vskip\fboxvsep
            \box\@tempboxa
            \vskip\fboxvsep}%
          #1%
          \vrule\@width\fboxrightrule}%
        \hrule\@height\fboxbottomrule}%
    }%
  }%
}
\long\def\fboxother#1{%
  \leavevmode
  \setbox\@tempboxa\hbox{%
    \color@begingroup
    \kern\fboxhsep{#1}\kern\fboxhsep
    \color@endgroup}%
  \@frameb@xother\relax}

\makeatother

\makeatletter
\def\@opargbegintheorem#1#2#3{\trivlist
\item[\hskip\dimexpr\labelsep+0pt
\relax{\bf #1\ #2}]({\bf #3})\, \itshape}
\makeatother

\firstpageno{1}

\usepackage{lastpage}
\jmlrheading{22}{2021}{1-\pageref{LastPage}}{7/20; Revised 4/21}{5/21}{20-721}{Uri Stemmer}
\ShortHeadings{Locally Private k-Means Clustering}{Stemmer}

\begin{document}

\title{Locally Private k-Means Clustering}

\author{\name Uri Stemmer \email u@uri.co.il \\
       \addr Ben-Gurion University, Beer-Sheva, Israel\\
			       Google Research, Tel Aviv, Israel}

\editor{Mehryar Mohri}
\maketitle

\begin{abstract}
We design a new algorithm for the Euclidean $k$-means problem that operates in the local model of differential privacy. Unlike in the non-private literature, differentially private algorithms for the $k$-means objective incur both additive and multiplicative errors. Our algorithm significantly reduces the additive error while keeping the multiplicative error the same as in previous state-of-the-art results. Specifically, on a database of size $n$, our algorithm guarantees $O(1)$ multiplicative error and $\approx n^{1/2+a}$ additive error for an arbitrarily small constant $a>0$. All previous algorithms in the local model had additive error $\approx n^{2/3+a}$. Our techniques extend to $k$-median clustering.

We show that the additive error we obtain is almost optimal in terms of its dependency on the database size $n$. Specifically, we give a simple lower bound showing that every locally-private algorithm for the $k$-means objective must have additive error at least $\approx\sqrt{n}$.
\end{abstract}

\begin{keywords}
  Differential privacy, local model, clustering, k-means, k-median.
\end{keywords}

\section{Introduction}

In center-based clustering, we aim to find a ``best'' set of centers (w.r.t.\ some cost function), and then partition the data points into clusters by assigning each data point to its nearest center. With over 60 years of research, center-based clustering is an intensively-studied key-problem in unsupervised learning (see \cite{hartigan1975clustering} for a textbook).  One of the most well-studied problems in this context is the {\em Euclidean $k$-means problem}. In this problem we are given a set of input points $S\subseteq\R^d$ and our goal is to identify a set $C$ of $k$ {\em centers} in $\R^d$, approximately minimizing the sum of squared distances from each input point to its nearest center. This quantity is referred to as the {\em cost} of the centers w.r.t.\ the set of points, denoted as $\cost_S(C)=\sum_{x\in S}\min_{c\in C}\|x-c\|^2$.

The huge applicability of $k$-means clustering, together with the increasing awareness and demand for user privacy, motivated a long line of research on {\em privacy preserving} $k$-means clustering. In this work we study the Euclidean $k$-means problem in the {\em local model of differential privacy (LDP)}.
Differentially private algorithms work in two main modalities: {\em trusted-curator} and {\em local}. The {\em trusted-curator} model assumes a trusted curator that collects all the personal information and then analyzes it. 
The privacy guarantee in this model is that the {\em outcome} of the analysis ``hides'' the information of any single individual, but this information is not hidden from the trusted curator. In contrast, the {\em local} model of differential privacy, which is the model we consider in this work, does not involve a trusted curator. In this model, there are $n$ users and an untrusted server, where each user $i$ is holding a private input item $x_i$ (a point in $\R^d$ in our case), and the server's goal is to compute some function of the inputs (approximate the $k$-means objective in our case). However, in this model, the users do not send their data as is to the server. Instead, every user randomizes her data locally, and only sends noisy reports to the server, who aggregates all the reports. Informally, the privacy requirement is that the input of user $i$ has almost no effect on the distribution on the messages that user $i$ sends to the server. We refer to the collection of user inputs $S = (x_1,\dots ,x_n)$ as a ``distributed database'' (as it is not stored in one location, and every $x_i$ is only held locally by user $i$). This model is used in practice by large corporations to ensure that private data never reaches their servers in the clear.

\begin{table*}[t]
\begin{adjustbox}{center}
\bgroup
\def\arraystretch{1.5}
\begin{tabular}{|c||c|c|c|}
\hline
{\bf Reference} & {\bf \# Rounds} & {\bf Multiplicative Error} & {\bf Additive Error}\\
\hline
\hline
\cite{NS18_1Cluster} & $O(k\log n)$ & $O(k)$ & $\tilde{O}\left(  n^{2/3+a} \cdot k^{4/3} \cdot d^{1/3} \right)$  \\
\hline
\cite{KaplanS18} & $O(1)$ & $O(1)$ & $\tilde{O}\left( n^{2/3+a} \cdot k^2 \cdot d^{1/3}\right)$ \\
\hline
\hline
This work & $O(1)$ & $O(1)$ & $\tilde{O}\left(n^{1/2+a}\cdot k \cdot \sqrt{d}\right)$ \\
\hline
\end{tabular}
\egroup
\end{adjustbox}
  \vspace{-5px}\caption{{\small\sffamily{Locally-private algorithms for $k$-means in the $d$-dimensional Euclidean space. Here $n$ is the number of input points, $k$ is the number of desired centers, $d$ is the dimension, and $a>0$ is an arbitrarily small constant (the constant hiding in the multiplicative error depends on $a$). We assume that input points come from the unit ball. For simplicity, we use the $\tilde{O}$ notation to hide logarithmic factors in $k,n,d$, as well as the dependency on the privacy parameters $\eps,\delta$ and the failure probability $\beta$.}}}\label{tbl:intro}
\end{table*}

As minimizing the $k$-means objective objective is NP-hard (even without privacy constraints), the literature has focused on approximation algorithms, with the current (non-private) state-of-the-art construction of \cite{AhmadianNSW17} achieving a multiplicative error of 6.357. That is, the algorithm of \cite{AhmadianNSW17} identifies a set of $k$ centers whose cost is no more than $6.357\cdot\OPT_S(k)$, where $\OPT_S(k)$ denotes the lowest possible cost.
Unlike in the non-private literature, it is known that every {\em differentially private} algorithm for approximating the $k$-means objective must have an {\em additive} error, which scales with the diameter of the input space~\cite{KaplanS18}. This is true both in the local model and in the trusted-curator model, even for computationally unbounded algorithms. Hence, a standard assumption for private $k$-means is that the input points come from the $d$-dimensional ball of radius $\Lambda$ around the origin $\BBB(0,\Lambda)$. This is the setting we consider in this work, where we assume that $\Lambda=1$ in the introduction.

There has been a significant amount of work aimed at constructing differentially private $k$-means algorithms that work in the {\em trusted-curator} model.\footnote{\cite{BDMN05,NRS07,FFKN09,McSherry09,GuptaLMRT10,Mohan2012,Wang2015,NockCBN16,Su2016,NSV16,DannyPrivatekMeans,Balcan17a,NS18_1Cluster,HuangL18,KaplanS18,ShechnerSS20,Ghazi0M20,CohenKMST21}} The current state-of-the-art construction by \cite{Ghazi0M20} obtains an $O(1)$ multiplicative error (which can be made arbitrarily close to the best non-private error) and $\poly(\log(n),k,d)$ additive error. That is, given a set of $n$ input points $S\in(\R^d)^n$, the algorithm of~\cite{Ghazi0M20} privately identifies a set of $k$ centers with cost at most $O(1)\cdot\OPT_S(k)+\poly(\log(n),k,d)$.

On the other hand, for the {\em local} model of differential privacy, only two constructions are available (with provable utility guarantees). The first construction, by \cite{NS18_1Cluster}, obtains $O(k)$ multiplicative error and $\approx n^{2/3
}$ additive error.\footnote{The error bounds stated throughout the introduction are simplified; see Table~1 for more details.} In addition to the relatively large multiplicative and additive errors, another downside of the algorithm of \cite{NS18_1Cluster} is that it requires $O(k\cdot\log(n))$ rounds of {\em interaction} between the users and the untrusted server. Following the work of \cite{NS18_1Cluster}, an improved locally-private $k$-means algorithm was presented by \cite{KaplanS18}, which requires only $O(1)$ rounds of interaction and guarantees a multiplicative error of $O(1)$ and an additive error of $\approx n^{2/3}$. That is, the algorithm of \cite{KaplanS18} reduced the number of interaction rounds while at the same time reducing the multiplicative error to a constant. However, the additive error still remained large. 
In this work we reduce the additive error to $\approx\sqrt{n}$ while keeping all other complexities the same, i.e.,\ with $O(1)$ rounds of interaction and with $O(1)$ multiplicative error.\footnote{Throughout the introduction, we write $\approx\sqrt{n}$ and $\approx n^{2/3}$ to mean $O(n^{1/2+a})$ and $O(n^{2/3+a})$, respectively.}

We remark that additive error of $\sqrt{n}$ is what one would expect in the local model of differential privacy, as this turned out to be the correct dependency of the error in $n$ for many other problems, including the {\em heavy-hitters} problem, {\em median} estimations, answering {\em counting queries}, and more. Indeed, in Section~\ref{sec:lower} we show that every locally-private algorithm for the $k$-means %
 must have additive error $\Omega(\sqrt{n})$. Hence, our positive result is almost optimal in terms of the dependency of the additive error in the database size $n$.

\subsection{Existing Techniques}\label{sec:existingTech}

Before presenting the new ideas of this work, we need to understand the reasons for why the previous results only achieved an additive error of $\approx n^{2/3}$. To that end, we give here an informal overview of the construction of \cite{KaplanS18}. 
Let $S=(x_1,\dots,x_n)\in(\R^d)^n$ be a (distributed) database. %
At a high level, the algorithm of \cite{KaplanS18} can be summarized as follows:

\begin{center}
\noindent\fboxother{
\parbox{0.9\linewidth}{
\hfill
\begin{enumerate}[leftmargin=10pt,rightmargin=3pt,itemsep=5pt,topsep=0pt]
\item Privately identify a set of {\em candidate centers} $Y\subseteq\R^d$ that contains a subset $Y^*\subseteq Y$ of size $k$ with low cost, say $\cost_{S}(Y^*)\leq O(1)\cdot\OPT_S(k)+\Gamma$ for some error parameter $\Gamma$.
	\item For every $y\in Y$, let $\#_S(y)$ denote the number of input points $x_i\in S$ such that $y$ is their nearest candidate center, that is 
	\vspace{-5pt}
	$$\#_S(y)=|\{x\in S: y=\argmin_{y'\in Y}\|x-y'\|\}|,\vspace{-5pt}$$
	and let $\hat{\#}_S(y)$ be a noisy estimation of $\#_S(y)$, satisfying LDP.
	\item Post-process the set of candidate centers and the noisy counts to identify a set $C$ of $k$ centers that approximately minimizes 
	\vspace{-5pt}
	$$\cost_{Y,\hat{\#}}(C)=\sum\nolimits_{y\in Y}\hat{\#}_S(y)\cdot\min_{c\in C}\cdot\|y-c\|^2.\vspace{-20pt}$$	
\end{enumerate}
}}\\[0.5em]
{\small\bfseries\sffamily Figure 1:} 
{\small\sffamily{High level overview of the construction of \cite{KaplanS18}.}}
\end{center}

Step~2 is done using standard LDP counting tools (to be surveyed in Section~\ref{sec:prelim}). Step~1 is more involved, and we will elaborate on it later. For now, it suffices to say that for any $\sqrt{n}\lesssim\Gamma\lesssim n$, there is an LDP algorithm that is capable of identifying a set $Y$ of size $|Y|\approx n/\Gamma$ that contains a subset of $k$ centers $Y^*\subseteq Y$ such that $\cost_{S}(Y^*)\lesssim O(1)\cdot\OPT_S(k)+\Gamma$. 
The analysis then goes by arguing that for every set of $k$ centers $D$ we have that 
$$\cost_{Y,\hat{\#}}(D)\approx\cost_{Y,\#}(D)\approx\cost_S(D),$$ 
where $\cost_{Y,\#}(D)$ is the same as $\cost_{Y,\hat{\#}}(D)$ but with the ``true'' counts $\#_S(y)$ instead of the estimated counts $\hat{\#}_S(y)$. This means that the set $C$ computed in Step~3 also has a low $k$-means cost w.r.t.\ the input points $S$, and is hence a good output. The main question is how {\em tight} are these connections. %
Using the fact that there is a subset %
$Y^*\subseteq Y$ with low $k$-means cost w.r.t.\ $S$, one can %
show that the connection $\cost_{Y,\#}(D)\approx\cost_S(D)$ holds, informally, up to an additive error of $O(\Gamma)$.%

The difficulty lies in the connection $\cost_{Y,\hat{\#}}(D)\approx\cost_{Y,\#}(D)$. As we mentioned, it is known that estimating counts under LDP generally incurs an additive error of $\Theta(\sqrt{n})$ (ignoring the dependency on all other parameters). As a result, for every $y\in Y$ the estimation error $\left|\hat{\#}_S(y)-\#_S(y)\right|$ might be as big as $\sqrt{n}$. Moreover, when comparing $\cost_{Y,\hat{\#}}(D)$ to $\cost_{Y,\#}(D)$, the different noises ``add up''. To illustrate this point, let $D$ be a possible set of centers, and observe that
\begin{align}
\cost_{Y,\hat{\#}}(D)&=\sum_{y\in Y}\hat{\#}_S(y)\cdot\min_{d\in D}\|y-d\|^2\nonumber\\
&\lesssim \sum_{y\in Y}\left(\#_S(y)+\sqrt{n}\right)\cdot\min_{d\in D}\|y-d\|^2\nonumber\\
&\leq\cost_{Y,\#}(D)+|Y|\cdot\sqrt{n}.\label{eq:additiveError}
\end{align}
Actually, it can be shown that the additive error only increases proportionally to $\sqrt{|Y|\cdot n}$, because of noise cancellations (as the sum of $|Y|$ independent noises only scales with $\sqrt{|Y|}$). In any case, at least with this type of an analysis, the error in the connection $\cost_{Y,\hat{\#}}(D)\approx\cost_{Y,\#}(D)$ scales with $\sqrt{|Y|\cdot n}$. Recall that the error in the other connection $\cost_{Y,\#}(D)\approx\cost_S(D)$ scales with $\Gamma\approx \frac{n}{|Y|}$. That is, the error in one of the two connections grows with $|Y|$, and the error in the second connection decreases with $|Y|$. These two requirements balance at $|Y|\approx n^{1/3}$, which results in an additive error of $\sqrt{n^{1/3}\cdot n}=n/n^{1/3}=n^{2/3}$. 
In a nutshell, this is the main reason for the large additive error in the construction of \cite{KaplanS18}. The construction of \cite{NS18_1Cluster} suffered from similar issues (although their algorithm is different).

\subsection{Our Contributions}

The takeaway from the above discussion is that in order to obtain an algorithm with small additive error, for every $y\in Y$, it suffices to ensure that $\hat{\#}_S(y)$ approximates $\#_S(y)$ to within a constant {\em multiplicative} factor. Indeed, in such a case Inequality~(\ref{eq:additiveError}) would be replaced with
\begin{align*}
\cost_{Y,\hat{\#}}(D)&=\sum_{y\in Y}\hat{\#}_S(y)\cdot\min_{d\in D}\|y-d\|^2\\
&\leq \sum_{y\in Y}O(1)\cdot \#_S(y)\cdot\min_{d\in D}\|y-d\|^2\\
&= O(1)\cdot\cost_{Y,\#}(D),
\end{align*}
which is acceptable (since we are aiming for a construction with a constant multiplicative error anyways). 
Observe that, as our noisy estimations are accurate to within an additive error of $\approx\sqrt{n}$, for every $y\in Y$ such that $\#_S(y)\gtrsim \sqrt{n}$ we already have that $\hat{\#}_S(y)$ approximates $\#_S(y)$ to within a constant multiplicative factor. However, this is not the case for candidate centers $y\in Y$ such that $\#_S(y)\ll \sqrt{n}$, and there could be many such candidates.

To summarize our discussion so far, we would like to identify a set $Y$ of candidate centers that satisfies the following two conditions.
\begin{enumerate}[leftmargin=74pt,itemsep=1pt]
	\item[{\bf Condition 1:}] $\exists Y^*\subseteq Y$ of size $k$ such that $$\cost_S(Y^*)\lesssim O(1)\cdot\OPT_S(k)+\sqrt{n}.$$
	\item[{\bf Condition 2:}] $\forall y\in Y$ we have $\#_S(y)\gtrsim\sqrt{n}$.
\end{enumerate}
While the set of candidate centers $Y$ computed in the previous works of \cite{NS18_1Cluster,KaplanS18} is guaranteed to satisfy the first condition above, we do not have any guarantee w.r.t.\ the second condition.

\medskip
\paragraph{First Attempt.} One might try to achieve the second condition above by simply deleting every $y\in Y$ such that $\hat{\#}_S(y)\lesssim\sqrt{n}$. This would indeed mean that, after the deletions, for every $y\in Y$ we have that %
 $\#_S(y)$ is at least $\sqrt{n}$. However, this might break condition 1. To see how this could happen, suppose that $k=d=2$, and consider a collection of points $p_1,\dots,p_{\sqrt{n}}$ around the point $(1,0)$, where every two points $p_i,p_j$ are at pairwise distance $\approx\rho$ (infinitely small), and all of them are within distance $\approx\rho$ to the point $(1,0)$. Now consider a database containing $(n-n^{3/4})$ copies of the point $(0,0)$, and $n^{1/4}$ copies of every $p_i$ (so that $S$ is of size $n$). Now suppose that $Y=\{(0,0),(0,1),p_1,\dots,p_{\sqrt{n}}\}$. Since our count estimations are only accurate up to an error of $\sqrt{n}$, we will have that $\hat{\#}_S(0,0)\approx n-n^{3/4}$, and that $\hat{\#}_S(y)\approx0$ for every other candidate center in $Y$. Hence, if we were to delete every $y\in Y$ with a small estimated count, then we would be left only with the point $(0,0)$, that misses the cluster around $(0,1)$, and hence $\cost_S(Y)\gtrsim n^{3/4}$, even though $\OPT_S(k)\approx 0$.

\medskip
\paragraph{Resolution.}
To overcome this challenge, we revisit the way in which the set of candidate centers $Y$ is constructed. 
Recall that the construction of \cite{KaplanS18} (described in Figure~1) is oblivious to the way in which the set of candidate centers $Y$ is constructed (the only requirement is that $Y$ contains a subset $Y^*$ with low $k$-means cost). For our new construction, we will need to identify additional properties of these candidate centers, that are specific to the way in which they are constructed. 

In more details, 
the set of candidate centers $Y$ is constructed in $\log(n)$ iterations, where during the $i$th iteration we identify ``large'' subsets of input points (at least $\gtrsim\sqrt{n}$ points) that can be enclosed in a ball of radius $r=2^{-i}$, and add to $Y$ a (privacy preserving) estimation for the average of every such subset of clustered input points. As the previous works of \cite{NS18_1Cluster,KaplanS18} showed, the set $Y$ constructed in this process contains (w.h.p.)\ a subset of $k$ centers with low $k$-means cost. Informally, the additional property that we will leverage is that for every candidate center $y$ that is constructed during the iteration with parameter $r$ there are $\gtrsim\sqrt{n}$ input points within distance $\lesssim r$ to $y$, which is true by the way in which the candidate centers are constructed. 

We will say that such a candidate center $y$ was {\em created for} the radius $r$. Moreover, we will say that an input point $x\in S$ {\em created} $y$, if $x$ was one of the points that $y$ was computed as their noisy average. 
So every candidate center has ``a lot'' ($\gtrsim\sqrt{n}$) input points who created it. Observe that this still does not guarantee that the resulting set $Y$ is such that every $y\in Y$ has ``a lot'' of neighbors in $S$, which is what we really wanted. This can happen, e.g., if some (or all) of the points who created the candidate center $y$ have a different candidate center that is closer to them.

To overcome this issue, we assign input points to candidate centers in a different way -- not by assigning every input point to its nearest candidate center. Informally, when assigning an input point $x\in S$ to a candidate center we give a slight preference to the candidate centers that $x$ created. We then estimate the {\em weight} of every candidate center $y\in Y$ according to this new assignment (where the {\em weight} of a candidate center is the number of input points assigned to it). We show that these new weights still allow us to estimate the cost of every set of centers w.r.t.\ the input points. In addition, with these new weights, we show that it is possible to delete every candidate center whose weight is lower than $\approx\sqrt{n}$. While this deletion step might delete many centers from $Y$, and in particular, might even delete the best $k$ centers from $Y$, we show that for every deleted candidate center $y$ there is a sequence of alternative candidate centers $y_1,y_2,\dots,y_w$ such that each of them could be a ``good substitute'' for $y$ and such that at least one of them is {\em not} deleted from $Y$. We obtain the following theorem (the details are given in Sections~\ref{sec:reform} and~\ref{sec:main}; see Theorem~\ref{thm:mainmain} for the formal statement). 

\begin{mytheorem}[informal]
There exists an LDP algorithm such that the following holds. When executed on a (distributed) database $S$ containing $n$ points in the $d$-dimensional unit ball, the algorithm returns a set $K$ of $k$ centers such that with high probability we have
$$
\cost_S(K)\leq O(1)\cdot\OPT_S(k)+\tilde{O}\left(k\cdot\sqrt{d} \cdot n^{0.5+a}\right),
$$
where $a>0$ is an arbitrarily small constant (the constant hiding in the multiplicative error depends on $a$).
\end{mytheorem}

We remark that all of our techniques extend to $k$-median clustering (this will be made precise in the technical sections). 
In Section~\ref{sec:lower} we show that the additive error achieved by our construction is almost optimal. Specifically, we present a lower bound showing that every %
 LDP algorithm for the $k$-means must have additive error $\Omega(\sqrt{n})$. This lower bound follows from a simple reduction from a task (related to) {\em counting bits} to the task of approximating the $k$-means objective of the data, together with known lower bounds for counting bits under LDP. We obtain the following theorem (see Theorem~\ref{thm:lowerlower} for the formal statement).

\begin{mytheorem}[informal]
Every LDP algorithm for approximating the $k$-means objective of a (distributed) database of size $n$ must have additive error $\Omega(\sqrt{n})$.
\end{mytheorem}

\subsection{Followup Work}
In this work, we design a new locally-private algorithm for the Euclidean $k$-means and $k$-median problems. Our algorithms 
require a constant number of interaction rounds between the users and the untrusted server, and guarantee $O(1)$ multiplicative error and $\approx n^{1/2+a}$ additive error for an arbitrarily small constant $a>0$. Following our work, \cite{ChangGKM21} presented a different algorithm that uses only a single round of interaction and obtains improved multiplicative and additive error guarantees. In particular, the algorithm of \cite{ChangGKM21} obtains additive error $\tilde{O}\left(\sqrt{n}\right)$ rather than $O\left( n^{1/2+a} \right)$ as in our result.

\section{Preliminaries}\label{sec:prelim}

In $k$-means clustering we aim to partition $n$ points into $k$ clusters in which each point $x$ belongs to the cluster whose mean is closest to $x$. We will also consider the $k$-median clustering objective, where we aim to place centroids at the {\em median} of every cluster (rather than the mean). Formally, 
for a set of points $S\in(\R^d)^n$ and a set of centers $C\subseteq \R^d$, the {\em $k$-means cost} of $C$ w.r.t.\ the points $S$ is defined as
$$\cost^2_S(C)=\sum_{x\in S}\min_{c\in C}\|x-c\|^2,$$
and the {\em $k$-median cost} is defined as 
$$\cost^1_S(C)=\sum_{x\in S}\min_{c\in C}\|x-c\|.$$
For $p\in\{1,2\}$ and for a {\em weighted} set $S=\{(x_1,\alpha_1),\dots,(x_n,\alpha_n)\}\in(\R^d\times\R)^n$, the {\em weighted cost} is 
$$\cost^p_S(C)=\sum_{(x,\alpha)\in S} \alpha\cdot\min_{c\in C}\|x-c\|^p.$$
We use $\OPT^p_S(k)$ to denote the lowest possible cost of $k$ centers w.r.t.\ $S$. That is, 
$$\OPT^p_S(k) = \min_{C\subseteq \R^d,\; |C|=k}\{\cost^p_S(C)\}.$$

\subsection{Local Differential Privacy}

The local model of differential privacy was formally defined by \cite{DMNS06} and \cite{KLNRS08}. We give here the formulation presented by \cite{Vadhan2016}. 
Consider $n$ parties $P_1,\dots,P_n$, where each party is holding a data item $x_i$. We denote $S=(x_1,\dots,x_n)$ and refer to $S$ as a {\em distributed database}. A {\em protocol} proceeds in a sequence of rounds until all (honest) parties terminate. Informally, in each round, each party selects a message to be broadcast based on its input, internal coin tosses, and all messages received in previous rounds. The {\em output} of the protocol is specified by a deterministic function of the transcript of messages exchanged. 
For some $j\in[n]$, we consider an {\em adversary} controlling all parties other than $P_j$. Given a particular adversary strategy $A$, we write $\view_{A}((A\leftrightarrow(P_1,\dots,P_n))(S))$ for the random variable that includes everything that $A$ sees when participating in the protocol $(P_1,\dots,P_n)$ on input $S=(x_1,\dots,x_n)$.

\begin{mydefinition}[\cite{DMNS06,KLNRS08,BeimelNO08,Vadhan2016}]
A protocol $P=(P_1,\dots,P_n)$ satisfies $(\eps,\delta)$-local differential privacy (LDP) if, for every $j\in[n]$, for every adversary $A$ controlling all parties other than $P_j$, for every two datasets $S,S'$ that differ on $P_j$'s input (and are equal otherwise), the following holds for every set $T$:
$$
\Pr[\view_{A}((A\leftrightarrow(P_1,\dots,P_n))(S))\in T] \leq e^{\eps}\cdot\Pr[\view_{A}((A\leftrightarrow(P_1,\dots,P_n))(S'))\in T]+\delta.
$$
\end{mydefinition}

As is standard in the literature on local differential privacy, we will consider protocols in which there is a unique player, called {\em the server}, which has no inputs. All other players are called {\em users}. Typically, users do not communicate with other users directly, only with the server.

\subsubsection{Counting Queries and Histograms}

For a database $S=(x_1,\dots,x_n)\in X^n$ and a domain element $x\in X$, we use $f_S(x)$ to denote the multiplicity of $x$ in $S$, i.e., $f_S(x)=|\{x_i\in S : x_i=x\}|$. 
One of the most basic tasks in the local model of differential privacy is computing {\em histograms}, in which the goal is to estimate $f_S(x)$ for every domain element $x$.

\begin{mytheorem}[\cite{HsuKR12,BassilyS15,BNST17,BunNS18}]\label{thm:HH}
Fix $\beta,\eps\leq1$. There exists a non-interactive $(\eps,0)$-LDP protocol that operates on a (distributed) database $S\in X^n$
for some finite set $X$, and returns a mapping $\hat{f}:X\rightarrow\R$ such that the following holds. For every choice of $x\in X$, with probability at least $1-\beta$, we have that
$$\left|\hat{f}(x) - f_S(x)\right|\leq \frac{3}{\eps}\cdot\sqrt{n\cdot \log\left(\frac{4}{\beta}\right)}.$$
\end{mytheorem}

\subsubsection{Composition and Post-Processing}

We will later present algorithms (or protocols) that instantiate several differentially private algorithms (or protocols). We will use the following theorems.

\begin{theorem}[\cite{DMNS06}]
Let $\Pi$ be an $(\eps,\delta)$-LDP protocol, and let $\Pi'$ be a protocol that invokes $\Pi$ and output an arbitrary function of its output. Then $\Pi'$ is $(\eps, \delta)$-LDP.
\end{theorem}

\begin{theorem}[\cite{DKMMN06, DRV10}]\label{thm:composition3}
Let $\Pi$ be a protocol that consists of $k$ (adaptive) executions of $(\eps,\delta)$-LDP protocols. Then $\Pi$ is $(k\eps, k\delta)$-LDP.
\end{theorem}

We remark that stronger composition theorems exist (in terms of the dependency of the resulting privacy guarantees in $k$), and refer the reader to~\cite{DRV10} for more details.

\section{Candidates with Additional Properties}\label{sec:reform}
As we explained in the introduction, %
 the first step in our construction is to privately identify a %
 set $Y$ of {\em candidate centers}. 
Our construction makes use of an LDP tool for this task, called \texttt{GoodCenters}. This tool, in its original form, was presented by \cite{NSV16} for the {\em trusted-curator} model, and was refined and extended to the local model by \cite{NS18_1Cluster}. 
At a high level, algorithm \texttt{GoodCenters} takes a parameter $r$ and works by hashing input points using a {\em locality sensitive hash function}, that aims to maximize the probability of a collision for ``close'' items (within distance $\leq r$), while minimizing the probability of collision for ``far'' items (at distance $\gg r$). If there is a ball of radius $r$ that encloses ``a lot'' ($\gtrsim \sqrt{n}$) of input points, then we expect that ``a lot'' of them will be hashed into the same hash value, which would allow us to isolate them and estimate their average with small error.

We identify additional properties of algorithm \texttt{GoodCenters}, which will be useful in the following section. Our contribution here is mostly conceptual -- in identifying the necessary properties and in showing that they are achieved by the algorithm. Most of the technical details in the construction and in the analysis of \texttt{GoodCenters} have already appeared in the works of \cite{NS18_1Cluster,KaplanS18}. Therefore, here we only state the properties of the algorithm, and present the formal details in the appendix.  
Our modification to \texttt{GoodCenters} is captured by Item~1 in the following theorem.

\begin{mytheorem}[Algorithm \texttt{GoodCenters}~\cite{NSV16,NS18_1Cluster,KaplanS18}]\label{thm:GoodCenters}
For every two constants $a>b>0$ there exists a constant $c=c(a,b)$ such that the following holds. Let $\beta,\eps,\delta,n,d,\Lambda,r$ be such that $\Lambda/r\leq\poly(n)$ and such that $t\geq O\left( \frac{n^{0.5+a+b}\cdot \sqrt{d}}{\eps}\log (\frac{1}{\beta}) \log\left(\frac{dn}{\beta\delta}\right) \right)$. Algorithm {\rm \texttt{GoodCenters}} satisfies $(\eps,\delta)$-LDP. Furthermore, let $S=(x_1,\dots,x_n)$ be a distributed database where every $x_i$ is a point in the $d$-dimensional ball $\BBB(0,\Lambda)$, and let {\rm \texttt{GoodCenters}} be executed on $S$ with parameters $r,t,\beta,\eps,\delta$. Denote $M=4n^a \ln(\frac{1}{\beta})$. 
The algorithm outputs a partition $I_1,\dots,I_M\subseteq[n]$, hash functions $h_1,\dots,h_M$, lists of hash values $L_1,\dots,L_M$, and sets of centers $Y_1,\dots,Y_M$, where for every $m\in[M]$ and $u\in L_m$ the set $Y_m$ contains a center $\hat{y}_{m,u}$. %
In addition,
\begin{enumerate}%
	\item With probability at least $1-\beta$, for every $m\in[M]$ and every $u\in L_m$ we have 
	$$\left|\left\{ i\in I_m: 
	\begin{array}{c}
	h_m(x_i)=u\\
	\text{and}\\
	\|x_i-\hat{y}_{m,u}\|\leq5cr
\end{array}	
	\right\}\right|\geq\frac{t\cdot n^{-b}}{16M}.$$
	\item Denote $Y=\bigcup_{m\in[M]}Y_m$. Then $$|Y|\leq\frac{512 \cdot n^{1+a+b}}{t}\ln\left(\frac{1}{\beta}\right).$$
	\item Let $P\subseteq S$ be a set of $t$ points that can be enclosed in a ball of radius $r$. With probability at least $1-\beta$ there exists $\hat{y}\in Y$ such that the ball of radius $5cr$ around $\hat{y}$ contains all of $P$. 
\end{enumerate}
\end{mytheorem}

\section{Algorithm \texttt{WeightedCenters}}\label{sec:main}

In this section we present our main construction -- algorithm \texttt{WeightedCenters}. In order to identify the set $Y$ of candidate centers, the algorithm begins by executing algorithm \texttt{GoodCenters} on the (distributed) database $S$ multiple times with exponentially growing choices for the parameter $r$. 
Recall that an execution of \texttt{GoodCenters} with parameter $r$ returns a partition $I^r_1,\dots,I^r_M\subseteq[n]$, hash functions $h^r_1,\dots,h^r_M$, lists of hash values $L^r_1,\dots,L^r_M$, and sets of centers $Y^r_1,\dots,Y^r_M$, where for every $m\in[M]$ and $u\in L^r_m$ the set $Y^r_m$ contains a center $\hat{y}^r_{m,u}$. By the properties of algorithm \texttt{GoodCenters}, with high probability, for every $m\in[M]$ and every $u\in L^r_m$ we have that
$$\left|\left\{i\in I^r_m: 
\begin{array}{c}
h^r_m(x_i)=u\\
\text{ and }\\
\|x_i-\hat{y}^r_{m,u}\|\leq5cr	
\end{array}
\right\}\right|\geq\frac{t}{16M}\cdot n^{-b}.$$
We introduce the following notation.

\begin{notation}\label{notation:create}
Given the outcomes of \texttt{GoodCenters} (with parameter $r$), we say that a point $x_i\in S$ (or, alternatively, that the $i^{\text{th}}$ user) {\em creates} a center $\hat{y}^r_{m,u}\in Y^r_m$ if $i\in I^r_m$ and $h^r_m(x_i)=u$ and $\|x_i-\hat{y}^r_{m,u}\|\leq5cr$. Observe that a point $x_i\in S$ creates at most one center in $Y^r=Y^r_1\cup\dots\cup Y^r_M$. If $x_i\in S$ creates a center in $Y^r$, then we say that $x_i$ {\em creates a center for the radius $r$}.
\end{notation}

\begin{myremark}
Algorithm \texttt{GoodCenters} constructs the centers in $Y^r$ by averaging (with noise) subsets of input points. Informally, we think of the set of points who ``create'' a center $\hat{y}\in Y^r$ as the set of points s.t.\ $\hat{y}$ was computed as their (noisy) average. The actual definition, however, is a bit different (as stated above).
\end{myremark}

After the set $Y$ of candidate centers is constructed, algorithm \texttt{WeightedCenters} proceeds by assigning input points to the centers in $Y$ (in a certain way) and estimating the {\em weight} of every candidate center  (where the {\em weight} of a candidate center is the number of input points assigned to it). Then, the algorithm {\em re-assigns} the input points to the candidate canters, and {\em re-estimates} the weights of the candidates. This second iteration of re-assigning points to centers (and re-estimating candidate weights) is done in order to ``eliminate'' candidates with low weights. Finally, the algorithm post-processes the weighted set of candidate centers in order to produce the final set of centers (this post-processing is done with a non-private algorithm for approximating either the $k$-means or the $k$-median cost objectives, as required).

\begin{algorithm*}[!htb]

\caption{\texttt{WeightedCenters}}\label{alg:WeightedCenters}

{\bf Input:} Failure probability $\beta$, privacy parameters $\eps,\delta$.

\smallskip
\noindent {\bf Setting: }Each player $j\in[n]$ holds a value $x_j\in \BBB(0,\Lambda)$. Define $S=(x_1,\dots,x_n)$.

\begin{enumerate}[leftmargin=15pt,rightmargin=10pt,itemsep=1pt,topsep=1.5pt]

\item[{\bf 1.}] {\bf \em Constructing candidate centers:} Denote $$t=O\left( \frac{1}{\eps}\cdot n^{0.5+a+b}\cdot \sqrt{d}\cdot\log(\frac{\log n}{\beta}) \log\left(\frac{dn}{\beta\delta}\right) \right).$$ For $r=\Lambda,\frac{\Lambda}{2},\frac{\Lambda}{4},\dots,\frac{\Lambda}{n}$, execute (in parallel) algorithm \texttt{GoodCenters} on $S$ with the parameter $t$ and the radius $r$ to obtain sets of centers $Y^r_1,\dots,Y^r_M$, lists $L^r_1,\dots,L^r_M$, hash functions $h^r_1,\dots,h^r_M$, and a partition $I^r_1,\dots,I^r_M\subseteq[n]$. Each execution of \texttt{GoodCenters} is done with privacy parameters $\frac{\eps}{4\log(n)},\frac{\delta}{\log(n)}$. Denote $Y=\bigcup_{r,m}Y^r_m$.

\item[{\bf 2.}] {\bf \em Assigning points to candidate centers:} 
Define the following assignment of users to centers in $Y$, denoted as $a(i,x_i)$. To compute $a(i,x_i)$, let $r_i\in\{\Lambda,\frac{\Lambda}{2},\dots,\frac{\Lambda}{n}\}$ be the smallest such that $x_i$ {\em creates} a center for $r_i$ (see Notation~\ref{notation:create}), and let $y(i,x_i,r_i)$ denote this created center. Then, let $y^*_i$ be a center with minimal distance to $x_i$ from $\bigcup_{\substack{r<r_i\\m\in[M]}}Y^r_m$. Now, if $\|x_i-y^*_i\|<\left\|x_i-y(i,x_i,r_i)\right\|$ then $a(i,x_i)=y^*_i$, and otherwise $a(i,x_i)=y(i,x_i,r_i)$.
\begin{enumerate}[leftmargin=30pt,rightmargin=0pt,itemsep=1pt,topsep=1.5pt]
	\item[{\gray{\%}}] \gray{Observe that each user $i$ can compute $a(i,x_i)$ herself from $i,x_i$ and from the publicly released sets of centers $Y^r_1,\dots,Y^r_M$, lists $L^r_1,\dots,L^r_M$, hash functions $h^r_1,\dots,h^r_M$, and partition $I^r_1,\dots,I^r_M\subseteq[n]$.}
\end{enumerate}

\item[{\bf 3.}] {\bf \em Estimating weights of candidate centers:} Use an $\frac{\eps}{4}$-LDP algorithm for histograms (see Theorem~\ref{thm:HH}) to obtain for every $y\in Y$ an estimation $$\hat{a}(y)\approx a(y)\triangleq|\{i: a(i,x_i)=y\}|.$$

\item[{\bf 4.}] {\bf \em Re-assigning points to candidate centers:} Let $$W=\left\{y\in Y : \hat{a}(y)\geq \Omega\left( \frac{\sqrt{d}}{\eps}\cdot n^{0.5+a}\cdot\log(\frac{1}{\beta}) \log\left(\frac{dn}{\delta}\right) \right) \right\},$$ and define $b(i,x_i)=a(i,x_i)$ if $a(i,x_i)\in W$, and otherwise define $b(i,x_i)$ to be an arbitrary center in $W$ with minimal distance to $x_i$.

\item[{\bf 5.}] {\bf \em Re-estimating weights of candidate centers:} Use an $\frac{\eps}{4}$-LDP algorithm for histograms (see Theorem~\ref{thm:HH}) to obtain for every $y\in W$ an estimation $$\hat{b}(y)\approx b(y)\triangleq|\{i: b(i,x_i)=y\}|.$$

\item[{\bf 6.}] {\bf \em Output:} Non-privately identify a subset $K\subseteq \R^d$ of size $k$ with low cost w.r.t.\ the set $W$ and the weights $\hat{b}$ (specifically, with cost at most $O(1)$ times the lowest possible cost).

\end{enumerate}
\end{algorithm*}

Consider the execution of \texttt{WeightedCenters} on a database $S$. For $p\in\{1,2\}$, let $C^{\opt_p}=(c^{\opt_p}_1,\dots,c^{\opt_p}_k)$ denote an optimal set of centers for $S$, where $p=1$ corresponds to the $k$-median objective and $p=2$ corresponds to the $k$-means objective. Also, for $p\in\{1,2\}$, let $S^{\opt_p}_1,\dots,S^{\opt_p}_k$ be the partition of $S$ induced by these optimal clusters. For $\ell\in[k]$ let 
$$r^{\opt_1}_{\ell}=\frac{2}{|S^{\opt_1}_{\ell}|}\sum_{x\in S^{\opt_1}_{\ell}}\|x-c^{\opt_1}_{\ell}\|,$$
and
$$r^{\opt_2}_{\ell}=\sqrt{\frac{2}{|S^{\opt_2}_{\ell}|}\sum_{x\in S^{\opt_2}_{\ell}}\|x-c^{\opt_2}_{\ell}\|^2}.$$
Now let $P^{\opt_p}_{\ell}=S^{\opt_p}_{\ell}\cap\BBB(c^{\opt_p}_{\ell},r^{\opt_p}_{\ell})$. Note that for every $\ell\in[k]$ we have that $|P^{\opt_p}_{\ell}|\geq\frac{1}{2}|S^{\opt_p}_{\ell}|$, as otherwise less than half of the points in $S^{\opt_p}_{\ell}$ are within distance $r^{\opt_p}_{\ell}$ to $c^{\opt_p}_{\ell}$, and so $\cost^p_{S^{\opt_p}_{\ell}}(\{c^{\opt_p}_{\ell}\})>\frac{|S^{\opt_p}_{\ell}|}{2}\cdot(r^{\opt_p}_{\ell})^p=\cost^p_{S^{\opt_p}_{\ell}}(\{c^{\opt_p}_{\ell}\})$.

\begin{notation}
We say that an optimal cluster $S^{\opt_p}_{\ell}$ is {\em large} if
$|S^{\opt_p}_{\ell}|\geq \tilde{O}\left( \frac{1}{\eps}\cdot n^{0.5+a+b}\cdot \sqrt{d}\right).$
\end{notation}

We begin the utility analysis by defining the following two events. The first event states that the executions of \texttt{GoodCenters} (in Step~1 of \texttt{WeightedCenters}) succeed. Specifically, the set of candidate centers $Y$ (resulting from Step~1) contains a ``close enough'' center for every {\em large} optimal cluster, and in addition, for every $y\in Y$ there are ``a lot'' of users who {\em created} $y$.

\begin{center}
\noindent\fboxother{
\parbox{.9\columnwidth}{
{\bf Event \texttt{CREATION} (over the executions of \texttt{GoodCenters}):}
\begin{enumerate}[itemsep=0pt,topsep=4pt]
	\item $\forall y\in Y=\bigcup_{r,m}Y^r_m$, there are at least $\tilde{t}=\Omega\left( \frac{\sqrt{nd}}{\eps}\log\left(\frac{dn}{\beta\delta}\right) \right)$ users that create $y$. 
	\item For every {\em large} optimal cluster $S^{\opt_p}_{\ell}$, the set $Y$ contains a center $y^*_{\ell}\in Y$ that was created for a radius $r^*_{\ell}$
	such that $r^*_{\ell}\leq\max\{2 r^{\opt_p}_{\ell},\frac{\Lambda}{n}\}$ and $\|y^*_{\ell}-c^{\opt_p}_{\ell}\|\leq O(r^*_{\ell})$.
\end{enumerate}
}}
\end{center}

\begin{claim}
Event \texttt{CREATION} occurs with probability at least $1-\beta$.
\end{claim}

\begin{proof}
Item~1 follows directly from the properties of algorithm \texttt{GoodCenters} and a union bound over the different choices for $r$. For item~2, fix $\ell\in[k]$ such that $S^{\opt_p}_{\ell}$ is large, and let $r^*_{\ell}\in\{\Lambda,\frac{\Lambda}{2},\dots,\frac{\Lambda}{n}\}$ be the smallest such that $r^*_{\ell}\geq r^{\opt_p}_{\ell}$. 
Note that $r^*_{\ell}\leq\max\{2 r^{\opt_p}_{\ell},\frac{\Lambda}{n}\}$. By Theorem~\ref{thm:GoodCenters}, the execution of \texttt{GoodCenters} with the radius $r^*_{\ell}$ (during Step~1 of algorithm \texttt{WeightedCenters}) identifies a center $y^*_{\ell}$ s.t.\ $\|y^*_{\ell}-c^{\opt_p}_{\ell}\|\leq O(r^*_{\ell})$ with probability at least $1-\frac{\beta}{k}$. By a union bound, with probability at least $1-\beta$, this happens for every large cluster $S^{\opt_p}_{\ell}$.
\end{proof}

The next event states that the two executions of LDP histograms (in Steps~3 and~5) succeed.

\begin{center}
\noindent\fboxother{
\parbox{.9\columnwidth}{
{\bf Event \texttt{HISTOGRAMS} (over the randomness in Steps~3 and~5):}\\ %
All the estimates computed in Steps~3 and~5 are accurate to within error $O\left(\frac{1}{\eps}\sqrt{n\log(\frac{n}{\beta})}\right)$.
}}
\end{center}

By Theorem~\ref{thm:HH}, Event \texttt{HISTOGRAMS} happens with probability at least $1-\beta$. 
We continue with the analysis assuming that Events \texttt{CREATION} and \texttt{HISTOGRAMS} occur. 
Recall that in Step~1 we generate the set of candidate centers $Y$ and that in Step~4 we define the subset $W\subseteq Y$ 
that contains only centers with ``large'' weights. The next two claims show that for every center $y\in Y$ there exists a center $w\in W$ that is ``close enough'' to $y$ (even if $y\notin W$). The first claim shows that if $y\in Y\setminus W$ then there is another center $y'\in Y$ that is close to $y$ (but $y'$ might also be missing from $W$). This will be leveraged in the claim that follows to identify a center {\em in $W$} that is close to $y$.

\begin{claim}\label{claim:InductionStep}
Let $y\in Y$ be a center that was created with the radius $r$. If $y\notin W$, then there is another center $y'\in Y$ that was created with a strictly smaller radius $r'<r$ such that $\|y-y'\|\leq O(r)$.
\end{claim}

\begin{proof}
By Event \texttt{CREATION}, there are at least $\tilde{t}=\Omega\left( \frac{\sqrt{nd}}{\eps}\log\left(\frac{dn}{\beta\delta}\right) \right)$ users who created the center $y$. Now, since $y\notin W$, it must be that for at least one user $i$ who created $y$, we have that $a(i,x_i)\neq y$, as otherwise $a(y)$ would be large and $y$ would be in $W$ (by Event \texttt{HISTOGRAMS}, the error in the estimation $\hat{a}(y)\approx a(y)$ is of a lower order). There could be two possible reasons for why $a(i,x_i)\neq y$:

\medskip
\paragraph{Case (a):}	User $i$ also created another center $y'$ for a smaller radius $r'<r$. In this case, since user $i$ created both $y$ and $y'$ we have that $\|x_i-y\|\leq O(r)$ and $\|x_i-y'\|\leq O(r')$, and hence $\|y-y'\|\leq O(r+r')=O(r)$ by the triangle inequality.

\medskip
\paragraph{Case (b):} User $i$ did not create a center for any radius smaller than $r$, but there is a center $y'$ created with radius $r'<r$ (that user $i$ did not create) such that $\|x_i-y'\|<\|x_i-y\|$. Since user $i$ did create $y$, we have that $\|x_i-y\|\leq O(r)$, and hence, we again have that $\|y-y'\|\leq O(r)$ by the triangle inequality.%
\end{proof}

The next claim applies the previous claim iteratively to identify a sequence of centers beginning from $y\in Y\setminus W$ and ending in a center $w\in W$ such that every two adjacent centers in this sequence are close. %

\begin{claim}\label{claim:ByInduction}
Let $y\in Y$ be a center that was created with the radius $r$. Then there is a center $w\in W$ such that $\|y-w\|\leq O(r)$.
\end{claim}

\begin{proof}
First observe that, by the definition of $a(\cdot,\cdot)$, if a user $i$ creates a center $y'$ for $r=\frac{\Lambda}{n}$ (the smallest possible radius) then $a(i,x_i)=y'$. Therefore, for every center $y'$ created with $r=\frac{\Lambda}{n}$ we have that $a(y')$ is large, and hence, $y'$ appears also in $W$. Now consider a center $y\in Y$ that was created with the radius $r$. If $y\in W$ then the claim is trivial. Otherwise, by induction using Claim~\ref{claim:InductionStep}, there is a sequence of centers $y_1,y_2,\dots,y_w$ such that 
\begin{enumerate}
	\item $y_1=y$,
	\item $y_w\in W$,
	\item $\|y_1-y_2\|\leq O(r)$ and for $i\geq1$ we have $$\|y_i-y_{i+1}\|\leq O\left(2^{-i}\cdot r \right),$$
\end{enumerate}
where Item~3 holds since $y_2$ was created with a strictly smaller radius than $y_1$, and $y_3$ was created with a strictly smaller radius than $y_2$, and so on. Therefore,
$$\|y-y_w\|\leq O\left(r+\frac{r}{2}+\frac{r}{4}+\dots\right)=O(r).$$
\end{proof}

The next claim shows that the set $W$ %
 contains a subset of $k$ centers with low cost.

\begin{claim}\label{claim:subsetW}
$\exists W^*\subseteq W$ of size $|W^*|=k$ such that
$$\cost^p_S(W^*)\leq O(1)\cdot\OPT^p_S(k)+ \tilde{O}\left( \frac{k \sqrt{d}\Lambda^p}{\eps}\cdot n^{0.5+a+b}\right).$$
\end{claim}

\begin{proof}
Recall that, by Event \texttt{CREATION}, for every {\em large} optimal cluster $S^{\opt_p}_{\ell}$, the set $Y$ contains a center $y^*_{\ell}\in Y$, which was created for a radius $r^*_{\ell}$, such that 
$$\|y^*_{\ell}-c^{\opt_p}_{\ell}\|\leq O(r^*_{\ell})=O\left(\max\left\{r^{\opt_p}_{\ell}\;,\;\frac{\Lambda}{n}\right\}\right).$$
Let $y^*_1,\dots,y^*_k\in Y$ and $r^*_1,\dots,r^*_k$ denote the aforementioned centers and %
 radiuses %
 (ignoring small clusters). Now, by Claim~\ref{claim:ByInduction}, the set $W$ contains centers $w^*_1,\dots,w^*_k$ such that for every large cluster $S^{\opt_p}_{\ell}$ we have $\|w^*_{\ell}-y^*_{\ell}\|\leq O(r^*_{\ell})= O(\max\{r^{\opt_p}_{\ell},\frac{\Lambda}{n}\})$. Hence, by the triangle inequality we have that $\|w^*_{\ell}-c^{\opt_p}_{\ell}\|\leq O(\max\{r^{\opt_p}_{\ell},\frac{\Lambda}{n}\})$. Denote $W^*=\left\{w^*_1,\dots,w^*_k\right\}$. If it were the case that all of the clusters are large, then we would have that
\begin{align*}
\cost^p_S(W^*) &= \sum_{x\in S}\min_{w\in W^*}\|x-w\|^p \\
&= \sum_{\ell\in[k]}\sum_{x\in S^{\opt_p}_{\ell}}\min_{w\in W^*}\|x-w\|^p\\
&\leq \sum_{\ell\in[k]}\sum_{x\in S^{\opt_p}_{\ell}}\|x-w^*_{\ell}\|^p\\
&\leq \sum_{\ell\in[k]}\sum_{x\in S^{\opt_p}_{\ell}}O\left(\|x-c^{\opt_p}_{\ell}\|^p + \|c^{\opt_p}_{\ell}-w^*_{\ell}\|^p\right)\\
&=O(1)\cdot\OPT^p_S(k) + O\left(\frac{\Lambda^p}{n^p}\right) + O(1)\cdot\sum_{\substack{\ell\in[k]\\x\in S^{\opt_p}_{\ell}}}\left(r^{\opt_p}_{\ell}\right)^p\\
&=O(1)\cdot\OPT^p_S(k)  + O\left(\frac{\Lambda^p}{n^p}\right).
\end{align*}
Now, the cost of a small cluster is at most $\Gamma=O\left( \frac{\sqrt{d}\cdot\Lambda^p}{\eps}\cdot n^{0.5+a+b}\cdot\log(\frac{k}{\beta}) \log\left(\frac{dnk}{\beta\delta}\right)\log(n) \right)$, and there could be at most $k$ such small clusters. Taking them into account, we have that
$$
\cost^p_S(W^*)\leq O(1)\cdot\OPT^p_S(k)+ O\left( k\cdot\Gamma\right).
$$
\end{proof}

In Step~4 of \texttt{WeightedCenters} we define an assignment $b(\cdot,\cdot)$ of the input points to the centers in $W$. If this assignment would simply assign each point to its nearest center in $W$, then (as $W$ contains a good set of centers by the previous claim) this assignment would trivially have a low cost. However, the assignment $b(\cdot,\cdot)$ does not necessarily match every point to its nearest center. Nevertheless, as the next claim shows, this assignment still has low cost.

\begin{claim}\label{claim:WvsOPT}
$$
\sum_{i\in[n]}\left\|b(i,x_i)-x_i\right\|^p\leq O(1)\cdot\OPT^p_S(k)+O\left( \frac{k\sqrt{d}\cdot\Lambda^p}{\eps}\cdot n^{0.5+a+b}\cdot\log\left(\frac{k}{\beta}\right) \log\left(\frac{dnk}{\beta\delta}\right)\log(n)\right).
$$
\end{claim}

\begin{proof}
Fix a large optimal cluster $S^{\opt_p}_{\ell}$, let $c^{\opt_p}_{\ell}$ denote its center, and let $x_i\in S^{\opt_p}_{\ell}$. %
By Event \texttt{CREATION}, a center $y^*_{\ell}$ for the cluster $S^{\opt_p}_{\ell}$ is created with radius $r^*_{\ell}\leq\max\{2r^{\opt_p}_{\ell},\frac{\Lambda}{n}\}$ such that $\|y^*_{\ell}-c^{\opt_p}_{\ell}\|\leq O(r^*_{\ell})$. (But it is not necessarily the case that $x_i$ created $y^*_{\ell}$.) Let $y(x_i)\in Y$ denote the center with the smallest radius that was created by $x_i$, and let $r(x_i)$ denote the radius for which $y(x_i)$ was created. Note that $y(x_i)$ might not be in $W$. 
There are two cases:

\medskip
\paragraph{Case (a):}	$\boldsymbol{r(x_i)>r^*_{\ell}.}$ 
	Then $\|a(i,x_i)-x_i\|\leq\|x_i-y^*_{\ell}\|$, because $a(i,x_i)$ can take the value $y^*_{\ell}$ if it minimizes the distance to $x_i$. Now, we either have that $b(i,x_i)=a(i,x_i)$ if $a(i,x_i)\in W$, or else $b(i,x_i)$ is set to be the closest center in $W$ to $x_i$, denoted as $W(x_i)$. So 
\begin{align*}
\|x_i-b(i,x_i)\|&\leq\|x_i-a(i,x_i)\|+\|x_i-W(x_i)\|\\
&\leq\|x_i-y^*_{\ell}\|+\|x_i-W(x_i)\|\\
&\leq\|x_i-c^{\opt_p}_{\ell}\|+\|c^{\opt_p}_{\ell}-y^*_{\ell}\|+\|x_i-W(x_i)\|\\
&\leq\|x_i-c^{\opt_p}_{\ell}\|+O\left(r^{\opt_p}_{\ell}+\frac{\Lambda}{n}\right)+\|x_i-W(x_i)\|.
\end{align*}	

\medskip
\paragraph{Case (b):} $\boldsymbol{r(x_i)\leq r^*_{\ell}.}$ Then, as $x_i$ created $y(x_i)$,
$$\|x_i-a(i,x_i)\|\leq O(r(x_i))=O(r^*_{\ell})$$
Recall that $a(i,x_i)$ might be missing from $W$, and hence, 
$$\|x_i-b(i,x_i)\|\leq O\left(r^{\opt_p}_{\ell}+\frac{\Lambda}{n}\right) + \|x_i-W(x_i)\|.$$

\medskip
So, in any case, we have that 
$$\|x_i-b(i,x_i)\|\leq\|x_i-c^{\opt_p}_{\ell}\|+O\left(r^{\opt_p}_{\ell}+\frac{\Lambda}{n}\right)+\|x_i-W(x_i)\|.$$
Therefore,
\begin{align*}
\sum_{i\in[n]}\left\|b(i,x_i)-x_i\right\|^p &= \sum_{\ell\in[k]}\sum_{x_i\in S^{\opt_p}_{\ell}}\left\|b(i,x_i)-x_i\right\|^p\\
&\leq\hspace{-7pt} \sum_{\substack{\ell\in[k]\\x_i\in S^{\opt_p}_{\ell}}}\hspace{-7pt}O\left(\|x_i-c^{\opt_p}_{\ell}\|^p+(r^{\opt_p}_{\ell})^p+\frac{\Lambda^p}{n^p}+\|x_i-W(x_i)\|^p\right)\\
&=O(1)\cdot\OPT^p_S(k)+ O\left(\frac{\Lambda^p}{n^p}\right) + O(1)\cdot\cost^p_S(W)\\
&\leq O(1)\cdot\OPT^p_S(k) + O\left(\frac{\Lambda^p}{n^p}\right) + O(1)\cdot\cost^p_S(W^*)\\
&\leq O(1)\cdot\OPT^p_S(k)\\
&\qquad
 + O\left( \frac{k}{\eps}\cdot n^{0.5+a+b}\cdot \sqrt{d}\cdot\log\left(\frac{k}{\beta}\right) \log\left(\frac{dnk}{\beta\delta}\right)\log(n) \cdot\Lambda^p\right),
\end{align*}
where $W^*\subseteq W$ is a subset of size $|W^*|=k$ that minimizes $\cost^p_S(W^*)$, and where the last inequality follows from Claim~\ref{claim:subsetW}.
\end{proof}

Recall that we denote $b(w)\triangleq|\{i: b(i,x_i)=w\}|$, and that in Step~5 we obtain estimations $\hat{b}(w)\approx b(w)$ for every $w\in W$. We write $B$ to denote the set $W$ with weights $\{b(w)\}$. That is, $B$ is a multiset of points containing $b(w)$ copies of every $w\in W$. Alternatively, $B$ is the multiset $B=\{b(i,x_i):i\in[n]\}$. We also write $\hat{B}$ to denote the set $W$ with the noisy weights $\{\hat{b}(w)\}$. These noisy weights might not be integers (and, in principle, could also be negative, but this does not happen when Event \texttt{HISTOGRAMS} occurs). 
The next claim shows that for every set of centers $D$ we have that $\cost^p_S(D)\approx\cost^p_B(D)$.

\begin{claim}\label{claim:SvsB}
Denote $$\Gamma=\frac{k}{\eps}\cdot n^{0.5+a+b}\cdot \sqrt{d}\cdot\log\left(\frac{k}{\beta}\right) \log\left(\frac{dnk}{\beta\delta}\right)\log(n) \cdot\Lambda^p.$$
For every set of centers $D\subseteq\R^d$ we have 
$$
\cost^p_B(D)\leq 2\cdot\cost^p_S(D)+ O(1)\cdot\OPT^p_S(k)+O\left( \Gamma\right),$$
and,
$$\cost^p_S(D)\leq 2\cdot\cost^p_B(D)+ O(1)\cdot\OPT^p_S(k)+O\left( \Gamma\right).$$
\end{claim}

\begin{proof}
For a set of centers $C$ and a point $x$ we write $C(x)$ to denote the closest neighbor of $x$ in $C$. By Claim~\ref{claim:WvsOPT}, for any set of centers $D\subseteq\R^d$ we have,
\begin{align*}
\cost^p_B(D)&=\sum_{i\in[n]}\left\|b(i,x_i)-D\big(b(i,x_i)\big)\right\|^p\\
&\leq\sum_{i\in[n]}\left\|b(i,x_i)-D(x_i)\right\|^p\\
&\leq\sum_{i\in[n]}\left(2\cdot\left\|b(i,x_i)-x_i\right\|^p+2\cdot\left\|x_i-D(x_i)\right\|^p\right)\\
&\leq O(1)\cdot\OPT^p_S(k)+O\left(\Gamma\right) +2\cdot\cost^p_S(D),
\end{align*}
and similarly,
\begin{align*}
\cost^p_S(D)&=\sum_{i\in[n]}\|x-D(x)\|^p\\
&\leq\sum_{i\in[n]}\|x-D\big(b(i,x_i)\big)\|^p\\
&\leq\sum_{i\in[n]}\left(2\left\|x-b(i,x_i)\right\|^p+2\left\|b(i,x_i)-D\big(b(i,x_i)\big)\right\|^p\right)\\
&\leq O(1)\cdot\OPT^p_S(k)+O\left(\Gamma\right) +2\cdot\cost^p_B(D).
\end{align*}
\end{proof}

The next claim shows that for every set of centers $D$ we have that $\cost^p_B(D)\approx\cost^p_{\hat{B}}(D)$.

\begin{claim}\label{claim:BvsHATB}
For every set of centers $D\subseteq\R^d$ we have 
$\frac{1}{2}\cost^p_B(D)\leq\cost^p_{\hat{B}}(D)\leq2\cost^p_B(D)$.
\end{claim}

\begin{proof}
First observe that, by Event \texttt{HISTOGRAMS}, for every $w\in W$ we have that $\frac{1}{2}b(w)\leq\hat{b}(w)\leq2 b(w)$.
To see this, note that by the definition of the set $W$ in Step~4, for every $w\in W$ we have that 
$$b(w)\geq \Omega\left( \frac{1}{\eps}\cdot n^{0.5+a}\cdot \sqrt{d}\cdot\log\left(\frac{1}{\beta}\right) \log\left(\frac{dn}{\delta}\right) \right).$$ 
In addition, by Event \texttt{HISTOGRAMS}, for every $w\in W$ we have that $$|b(w)-\hat{b}(w)|\leq O\left(\frac{1}{\eps}\sqrt{n\cdot\log(\frac{n}{\beta})}\right)\ll b(w),$$ and hence, for every $w\in W$ we have $\frac{1}{2}b(w)\leq\hat{b}(w)\leq2 b(w)$. Now let $D\subseteq\R^d$ be a set of centers. We have that
\begin{align*}
\cost^p_B(D)&=\sum_{w\in W}b(w)\cdot\min_{d\in D}\|w-d\|^p\\
&\leq\sum_{w\in W}2\cdot\hat{b}(w)\cdot\min_{d\in D}\|w-d\|^p\\
&=2\cdot \cost^p_{\hat{B}}(D).
\end{align*}
The other direction is symmetric.
\end{proof}

So, by the last two claims, for every set of centers $D$ we have that $\cost^p_S(D)\approx\cost^p_B(D)\approx\cost^p_{\hat{B}}(D)$. Hence, we can use the (privately computed) weighted set $\hat{B}$ as a proxy in order to identify $k$ centers with low cost w.r.t.\ $S$. This is formalized in the following theorem.

\begin{mytheorem}\label{thm:mainmain}
Let $p\in\{1,2\}$. 
Algorithm \texttt{WeightedCenters} satisfies $(\eps,\delta)$-LDP. In addition, when executed on a (distributed) database $S$ containing $n$ points in the $d$-dimensional ball $\BBB(0,\Lambda)$, the algorithm returns a set $K$ of $k$ centers such that with probability at least $1-O(\beta)$ we have
$$
\cost^p_S(K)\leq O(1)\cdot\OPT^p_S(k)+\tilde{O}\left(\frac{k\sqrt{d}\Lambda^p}{\eps}
n^{0.5+a+b}\right),
$$
where $a>b>0$ are arbitrarily small constants (the constant hiding in the multiplicative error depends on $a$ and $b$).
\end{mytheorem}

\begin{proof}
The privacy properties of \texttt{WeightedCenters} are straightforward (follow from composition and post-processing). We proceed with the utility analysis. Let $C^{\opt_p}_S\subseteq \R^d$ be a subset of $k$ centers minimizing $\cost^p_S(\cdot)$, and let $C^{\opt_p}_{\hat{B}}\subseteq \R^d$ be a subset of $k$ centers minimizing $\cost^p_{\hat{B}}(\cdot)$. The output of algorithm \texttt{WeightedCenters} is a set $K\subseteq \R^d$ of size $|K|=k$ such that 
$\cost^p_{\hat{B}}(K)\leq O(1)\cdot\cost^p_{\hat{B}}(C^{\opt_p}_{\hat{B}})$. We now show that the set $K$ has low cost w.r.t.\ $S$. 
Denote $\Gamma=\frac{k}{\eps}\cdot n^{0.5+a+b}\cdot \sqrt{d}\cdot\log(\frac{k}{\beta}) \log\left(\frac{dnk}{\beta\delta}\right)\log(n) \cdot\Lambda^p$. By Claims~\ref{claim:SvsB} and~\ref{claim:BvsHATB},
\begin{align*}
\cost^p_S(K)&\leq O(1)\cdot\cost^p_{\hat{B}}(K)+ O(1)\cdot\OPT^p_S(k)+O(\Gamma).%
\end{align*}
Now, by the fact that $K$ approximately minimizes $\cost^p_{\hat{B}}(\cdot)$, the last expression is at most
\begin{align*}
&\qquad\leq O(1)\cdot\cost^p_{\hat{B}}(C^{\opt_p}_{\hat{B}})+ O(1)\cdot\OPT^p_S(k)+O(\Gamma).%
\end{align*}
Since $C^{\opt_p}_{\hat{B}}$ minimizes $\cost^p_{\hat{B}}(\cdot)$, the last expression is at most
\begin{align*}
&\qquad\leq O(1)\cdot\cost^p_{\hat{B}}(C^{\opt_p}_{S})+ O(1)\cdot\OPT^p_S(k)+O(\Gamma).%
\end{align*}
Finally, by using again Claims~\ref{claim:SvsB} and~\ref{claim:BvsHATB}, we get that the last expression is at most
\begin{align*}
\phantom{\cost^p_S(K)}&\leq O(1)\cdot\cost^p_{S}(C^{\opt_p}_{S})+ O(1)\cdot\OPT^p_S(k)+O(\Gamma)\\
&= O(1)\cdot\OPT^p_S(k)+O(\Gamma).
\end{align*}
\end{proof}

\section{A Lower Bound on the Additive Error}\label{sec:lower}

In this section we present a simple lower bound on the error of every LDP algorithm for approximating the $k$-means (or the $k$-median) cost objective. To get our lower bound, we show a reduction from the following problem, called {\em Gap-Threshold}, to the $k$-means problem, and then use an existing lower bound for the Gap-Threshold problem.

\begin{mydefinition}[\cite{BeimelNO08}]
For $\tau>0$ and $x_1,\dots,x_n\in\{0,1\}$,
$$
\GAPTR_{\tau}(x_1,\dots,x_n)=\left\{
\begin{array}{l}
0, \quad \text{If }\;\; \sum_{i\in[n]}x_i= 0\\
1, \quad \text{If }\;\; \sum_{i\in[n]}x_i\geq \tau
\end{array}
\right.
$$
Note that $\GAPTR_{\tau}(x_1,\dots,x_n)$ is not defined when $0<\sum_{i\in[n]}x_i<\tau$.
\end{mydefinition}

\begin{mytheorem}[\cite{BeimelNO08,JosephMaNeRo19}]\label{thm:BNO}
There exist constants $0<\eps,\beta<1$ such that the following holds.
Let $\delta=o\left(\frac{1}{n^2 \log n}\right)$ and let $\AAA$ be an $(\eps,\delta)$-LDP protocol for computing $\GAPTR_{\tau}$ with success probability $1-\beta$. Then $\tau=\Omega(\sqrt{n})$.
\end{mytheorem}

Theorem~\ref{thm:BNO} is stated in~\citep{BeimelNO08} for $(\eps,0)$-LDP protocols with a bounded number of interaction rounds. The extension to arbitrary LDP protocols follows from \cite[Theorem~5.3]{JosephMaNeRo19}. We now show that this theorem implies a lower bound of $\Omega(\sqrt{n})$ on the additive error of LDP algorithms for the $k$-means or the $k$-median, even when the dimension $d$ is 1 and $k=2$.

\begin{mytheorem}\label{thm:lowerlower}
There exists constant $0<\eps,\beta<1$ such that the following holds. 
Let $\delta=o\left(\frac{1}{n^2 \log n}\right)$, and let $\AAA$ be an $(\eps,\delta)$-LDP protocol such that for any (distributed) database $X\in([0,1])^n$,  with probability $1-\beta$ the algorithm outputs a set $C$ of $k=2$ centers satisfying $\cost^p_X(C)\leq\gamma\cdot\OPT^p_X(k)+\tau$, where $\gamma<\infty$ and $p\in\{1,2\}$. Then $\tau=\Omega(\sqrt{n})$.
\end{mytheorem}

{\floatname{algorithm}{Protocol}
\begin{algorithm*}[t]
\caption{$\BBB$}

\noindent {\bf Setting: }Each user $i\in[n]$ holds a bit $b_i\in\{0,1\}$. Define $S=(b_1,\dots,b_n)$.

\smallskip

\noindent {\bf Parameter:} $r=\beta/4$, where $\beta$ is the constant from Theorem~\ref{thm:BNO}.

\begin{enumerate}[leftmargin=15pt,rightmargin=10pt,itemsep=1pt,topsep=1.5pt]

\item[{\bf 1.}]{\bf The server:} Let $R=\{[0,r],\; [r,2r],\dots, [1-r,1]\}$, sample (uniformly) an interval $I\in R$, and let $\mu$ denote the center of $I$. Send $\mu$ to all the users.

\item[{\bf 2.}]{\bf Every user $\boldsymbol{i}$:} If $b_i=0$ then ignore $\mu$ and set $x_i=0$. Otherwise, set $x_i=\mu$.

\item[{\bf 2.}]{\bf The server and the users:} Execute protocol $\AAA$ on the database $X=(x_1,\dots,x_n)$ to obtain a set of centers $C=\{c_1,c_2\}$.

\item[{\bf 3.}]{\bf The server:} If $c_1\in I$ or $c_2\in I$ then return 1. Otherwise return 0.
\end{enumerate}
\end{algorithm*}}

\begin{proof}
Let $\beta,\eps$ be the constants from Theorem~\ref{thm:BNO}. Let $\AAA$ be an $(\eps,\delta)$-LDP protocol that operates on a (distributed) database $X\in([0,1])^n$ and outputs a set $C$ of size $k=2$. %
We use $\AAA$ to construct a protocol for $\GAPTR$, described in protocol $\BBB$.

By Theorem~\ref{thm:BNO}, there exists an error parameter $\tau=\Omega(\sqrt{n})$, and a database $S^*\in\{0,1\}^n$ such that $\BBB(S^*)\neq\GAPTR_{\tau}(S^*)$ with probability at least $\beta$. We now show that $S^*$ cannot be the all zero database. To that end, observe that if the input $S$ is the all zero database, then all of the users in the protocol $\BBB$ ignore $\mu$, and hence, algorithm $\AAA$ gets no information about the selected interval $I$. In that case, the probability that one of $c_1,c_2$ falls in $I$ is at most $2r$. That is, the probability that $\BBB(\vec{0})\neq0=\GAPTR_{\tau}(\vec{0})$ is at most $2r=\beta/2$. Therefore, the database $S^*$ (on which $\BBB$ errs with probability at least $\beta$) is {\em not} the all zero database $\vec{0}$. Hence, $S^*$ contains at least $\tau$ ones. Therefore, whenever $\BBB$ errs on $S^*$, we have that $\cost^p_X(C)\geq \left(\frac{r}{2}\right)^p\cdot\tau=\Omega(\sqrt{n})$, even though $\OPT^p_X(k)=0$.
This happens with probability at least $\beta$, which completes the proof.
\end{proof}

\acks{The author was partially supported by the Israel Science Foundation (grant No.\ 1871/19), 
and by the Cyber Security Research Center at Ben-Gurion University of the Negev. The author would like to thank Haim Kaplan and the anonymous reviewers for their helpful comments.}

\appendix

\hypersetup{bookmarksdepth=-1}
\section{ }

In this section we provide the details that were omitted from Section~\ref{sec:reform}, and prove Theorem~\ref{thm:GoodCenters}.

\subsection{Additional Preliminaries}

\subsubsection{Average of Vectors in $\R^d$ Under LDP}
  
Consider a (distributed) database $X=(x_1,\dots,x_n)$ where every user $i$ is holding $x_i\in\R^d$. One of the most basic tasks we can apply under local differential privacy is to compute a noisy estimation for the sum (or the average) of vectors in $X$. Specifically, every user sends the server a noisy estimation of its vector (e.g., by adding independent Gaussian noise to each coordinate), and the server simply sums all of the noisy reports to obtain an estimation for the sum of $X$.

\begin{mytheorem}[folklore]\label{thm:LDPsum}
Consider a (distributed) database $X=(x_1,\dots,x_n)$ where every user $i$ is holding a point $x_i$ in the $d$ dimensional ball $\BBB(0,\Lambda)$. There exists an $(\eps,\delta)$-LDP protocol for computing an estimation $a$ for the sum of the vectors in $X$, such that with probability at least $(1-\beta)$ we have
$$
\left\|a-\sum_{i\in[n]}x_i\right\|\leq\frac{2\Lambda\sqrt{nd}\ln(\frac{2}{\beta\delta})}{\eps}.
$$
\end{mytheorem}

For our constructions we will need a tool for computing averages of {\em subsets} of $X$. Specifically, assume that there are $n$ users, where user $i$ is holding a point $x_i\in\R^d$. Moreover, assume that we have a fixed (publicly known) partition of $\R^d$ into a finite number of regions: $R_1,\dots,R_T\subseteq\R^d$. For every region $R_{\ell}$, we would like to obtain an estimation for the average of the input points in that region. This can be done using the following simple protocol, called \texttt{LDP-AVG}.

{\floatname{algorithm}{Protocol}
\begin{algorithm*}[!htp]
\caption{\texttt{LDP-AVG}}

\noindent {\bf Public parameters:} Partition of $\R^d$ into $t$ regions $R_1,\dots,R_T$.

\smallskip

\noindent {\bf Setting: }Each user $i\in[n]$ holds a point $x_i\in\R^d$. Define $X=(x_1,\dots,x_n)$.

\begin{enumerate}[leftmargin=15pt,rightmargin=10pt,itemsep=1pt,topsep=1.5pt]

\item[{\bf 1.}]{\bf Every user $\boldsymbol{i}$:} Let $y_i=(y_{i,1},\dots,y_{i,T})\in(\R^d)^T$ be a vector whose every coordinate is an independent Gaussian noise. Specifically, every $y_{i,t}\in\R^d$ is a vector whose every coordinate is sampled i.i.d.\ from $N(0,\sigma_t^2)$, for $\sigma_t=\frac{8\cdot\diam(R_t)}{\eps}\sqrt{\ln(1.25/\delta)}$. Let $t$ be s.t.\ $x_i\in\R_t$. Add $x_i$ to $y_{i,t}$. Send $y_i$ to the server.

\item[{\bf 2.}]{\bf The server and the users:} Run the protocol from Theorem~\ref{thm:HH} with privacy parameter $\frac{\eps}{2}$. For every $t\in[T]$ the server obtains an estimation $\hat{r}_t\approx|\{i:x_i\in R_t\}|\triangleq r_t$.

\item[{\bf 3.}]{\bf The server:} Output a vector $\hat{a}\in(\R^d)^T$, where $\hat{a}_t=\frac{1}{\hat{r}_t}\cdot\sum_{i\in[n]}y_{i,t}$.
\end{enumerate}
\end{algorithm*}}

\begin{claim}\label{claim:LDP-AVG}
\texttt{LDP-AVG} satisfies $(\eps,\delta)$-LDP. Moreover, with probability at least $(1-\beta)$, for every $t\in[T]$ s.t.\ $r_t\geq\frac{12}{\eps}\cdot\sqrt{n\cdot \log\left(\frac{4T}{\beta}\right)}$ we have that 
$$
\left\|\frac{1}{\hat{r}_t}\cdot\sum_{i\in[n]}y_{i,t}-\frac{1}{r_t}\cdot\sum_{\substack{i\in[n]:\\x_i\in R_t}}x_i\right\|\leq\frac{36\sqrt{dn}\cdot\ln(\frac{8dT}{\beta\delta})}{\eps\cdot r_t}\cdot\diam(R_t).
$$
\end{claim}

\begin{proof}
The privacy properties of \texttt{LDP-AVG} follow from the privacy properties of the Gaussian mechanism and the protocol from Theorem~\ref{thm:HH}, together with composition. Observe that by Theorem~\ref{thm:HH}, with probability at least $(1-\frac{\beta}{2})$, for every $t\in[T]$ we have $|r_t-\hat{r}_t|\leq\frac{6}{\eps}\cdot\sqrt{n\cdot \log\left(\frac{4T}{\beta}\right)}$. We continue with the analysis assuming that this is the case. Fix $t\in[T]$ such that $r_t\geq\frac{12}{\eps}\cdot\sqrt{n\cdot \log\left(\frac{4T}{\beta}\right)}$, and observe that $\hat{r}_t\geq r_t/2$. Denote $\Delta_t=\diam(R_t)$. Using a standard tail bound for normal variables, with probability at least $1-\frac{\beta}{2T}$ we have that $\left\|\sum_{i\in[n]}y_{i,t}-\sum_{\substack{i\in[n]:\\x_i\in R_t}}x_i\right\|\leq\frac{12\sqrt{dn}\Lambda_t\cdot\ln(\frac{4dT}{\beta\delta})}{\eps}$. Hence,
\begin{align*}
\left\|\frac{1}{\hat{r}_t}\cdot\sum_{i\in[n]}y_{i,t}-\frac{1}{r_t}\cdot\sum_{\substack{i\in[n]:\\x_i\in R_t}}x_i\right\|
& \leq \left|\frac{1}{\hat{r}_t}-\frac{1}{r_t}\right|\cdot\left\|\sum_{i\in[n]}y_{i,t}\right\|
+\left\|\frac{1}{r_t}\cdot\left(\sum_{i\in[n]}y_{i,t}-\sum_{\substack{i\in[n]:\\x_i\in R_t}}x_i\right)\right\|
\end{align*}
\begin{align*}
&\qquad\leq \left|\frac{1}{\hat{r}_t}-\frac{1}{r_t}\right|\cdot\left(\frac{12\sqrt{dn}\Lambda_t\cdot\ln(\frac{4dT}{\beta\delta})}{\eps}+\left\|\sum_{\substack{i\in[n]:\\x_i\in R_t}}x_i\right\|\right)+\frac{12\sqrt{dn}\Lambda_t\cdot\ln(\frac{4dT}{\beta\delta})}{\eps\cdot r_t}\\
&\qquad= \left|\frac{1}{\hat{r}_t}-\frac{1}{r_t}\right|\cdot\frac{12\sqrt{dn}\Lambda_t\cdot\ln(\frac{4dT}{\beta\delta})}{\eps}+\left|\frac{1}{\hat{r}_t}-\frac{1}{r_t}\right|\cdot\left\|\sum_{\substack{i\in[n]:\\x_i\in R_t}}x_i\right\|+\frac{12\sqrt{dn}\Lambda_t\cdot\ln(\frac{4dT}{\beta\delta})}{\eps\cdot r_t}\\
&\qquad\leq \frac{1}{r_t}\cdot\frac{12\sqrt{dn}\Lambda_t\cdot\ln(\frac{4dT}{\beta\delta})}{\eps}+\frac{\left|r_t-\hat{r}_t\right|}{\left|r_t\cdot\hat{r}_t\right|}\cdot\sum_{\substack{i\in[n]:\\x_i\in R_t}}\left\|x_i\right\|+\frac{12\sqrt{dn}\Lambda_t\cdot\ln(\frac{4dT}{\beta\delta})}{\eps\cdot r_t}\\
&\qquad\leq \frac{24\sqrt{dn}\Lambda_t\cdot\ln(\frac{4dT}{\beta\delta})}{\eps\cdot r_t}+\frac{\left|r_t-\hat{r}_t\right|}{\left|r_t\cdot\hat{r}_t\right|}\cdot \Lambda_t\cdot r_t\\
&\qquad\leq \frac{24\sqrt{dn}\Lambda_t\cdot\ln(\frac{4dT}{\beta\delta})}{\eps\cdot r_t}+\frac{12\Lambda_t\cdot\sqrt{n\cdot \log\left(\frac{8T}{\beta}\right)}}{\eps\cdot r_t}\\
&\qquad\leq \frac{36\sqrt{dn}\Lambda_t\cdot\ln(\frac{8dT}{\beta\delta})}{\eps\cdot r_t}.
\end{align*}
The claim now follows from a union bound.
\end{proof}

\subsubsection{Random Rotation}

We also use the following technical lemma to argue that if a set of points $P$ is contained within a ball of radius $r$ in $\R^d$, then by randomly rotating the Euclidean space we get that (w.h.p.) $P$ is contained within an axis-aligned rectangle with side-length $\approx r/\sqrt{d}$.

\begin{mylemma}[e.g.,~\cite{VaziraniRao}]\label{lem:RandomRotation}
Let $P \in (\R^d)^m$ be a set of $m$ points in the $d$ dimensional Euclidean space, and let $Z=(z_1,\ldots,z_d)$ be a random orthonormal basis for $\R^d$. Then,
$$\Pr_Z \left[\forall x,y \in P:\; \forall 1\leq i\leq d: \; \left| \langle x-y,  z_i\rangle\right| \leq 2\sqrt{\ln(dm/\beta)/d}\cdot \|x-y\|\right]\geq1-\beta.$$
\end{mylemma}

\subsubsection{Locality Sensitive Hashing}

A locality sensitive hash function aims to maximize the probability of a collision for similar items, while minimizing the probability of collision for dissimilar items. Formally,

\begin{mydefinition}[\cite{OriginalLSH}]\label{def:lsh}
Let $\cal M$ be a metric space, and let $r{>}0$, $c{>}1$, $0{\leq}q{<}p{\leq}1$.
A family $\HHH$ of functions mapping $\cal M$ into domain $U$ is an $(r,cr,p,q)$ locality sensitive hashing family (LSH) if for all $x,y\in{\cal M}$ (i)
$\Pr_{h\in_R \HHH}[h(x)=h(y)] \geq p$ if $d_{\cal M}(x,y) \leq r$; and (ii) $\Pr_{h\in_R \HHH}[h(x)=h(y)] \leq q$ if $d_{\cal M}(x,y) \geq cr$.
\end{mydefinition}

\subsection{Algorithm \texttt{CentersProcedure}}

Before presenting algorithm \texttt{GoodCenters} and its analysis, we introduce the following procedure, called \texttt{CentersProcedure}, which is the main ingredient in the construction of algorithm \texttt{GoodCenters}. 
This procedure identifies a set of candidate centers that, with noticeable probability, ``captures'' every large enough cluster of input points. In algorithm \texttt{GoodCenters}, we will later apply \texttt{CentersProcedure} multiple times to boost the success probability. The privacy properties of \texttt{CentersProcedure} are immediate (follow from composition), and are specified in the following observation.

In more detail, algorithm \texttt{CentersProcedure} first samples a locality sensitive hash function $h$, identifies (using standard LDP tools for histograms) a list of hash values that are ``heavy'' in the sense that many of the users' input are mapped to these values. Then, for every such heavy hash value, the algorithm averages (with noise to ensure LDP) all the input points that are mapped (by $h$) to this hash value. In order to reduce the noise incurred by averaging, before applying the LDP averaging tool, the algorithm encloses each cluster of input points that corresponds to a heavy hash value in a small box (with random rotation).

\begin{algorithm*}[!t]

\caption{\texttt{CentersProcedure}}\label{alg:CentersProcedure}

{\bf Input:} Radius $r$, target number of points $t$, failure probability $\beta$, privacy parameter $\eps$.

\smallskip
\noindent {\bf Tool used:} Family $\HHH$ of $(r,c\cdot r,p,q)$-locality sensitive hash functions mapping $\R^d$ to a universe $U$.

\smallskip
\noindent {\bf Setting: }Each player $j\in[n]$ holds a value $x_j\in \BBB(0,\Lambda)$. Define $S=(x_1,\dots,x_n)$.

\begin{enumerate}[leftmargin=15pt,rightmargin=10pt,itemsep=1pt,topsep=1.5pt]

\item Sample a hash function $h\in\HHH$ mapping $\R^d$ to $U$.

\item Use Theorem~\ref{thm:HH} with $\frac{\eps}{4}$ to identify a list $L\subseteq U$ such that
\begin{enumerate}[leftmargin=35pt,rightmargin=10pt,itemsep=1pt,topsep=1.5pt]
	\item Every $u\in U$ s.t.\ $|\{x\in S : h(x)=u\}|\geq\frac{t}{16}\cdot n^{-b}$ is in $L$,
	\item For every $u\in L$ we have $|\{x\in S : h(x)=u\}|\geq\frac{t}{32}\cdot n^{-b}$,
	\item The list $L$ is of size at most $32 n^{1+b}/t$.
\end{enumerate}

\item\label{step:rotation} Let $Z=(z_1,\dots,z_d)$ be a random orthonormal basis of $\R^d$, and denote $p=2rc \sqrt{\ln(\frac{dn}{\beta})/d}$. Also let $\III=\{I_1,I_2,\dots, I_{4\Lambda/p}\}$ be a partition of $[-2\Lambda,2\Lambda]$ into intervals of length $p$.

\item Randomly partition $S$ into subsets $S^1,\dots,S^d$ of size $|S^i|=\frac{n}{d}$. For every basis vector $z_i\in Z$,
use Theorem~\ref{thm:HH} with $\frac{\eps}{4}$ to obtain for every pair $(I,u)\in\III{\times}U$ an estimation $a_i(I,u)$ for 
$$|\{x\in S^i :  h(x)=u \text{ and } \langle x,z_i\rangle\in I\}|.$$

\item For every basis vector $z_i\in Z$ and for every hash value $u\in L$, denote $I(i,u)=\argmax_{I\in\III}\{ a_i(I,u) \}$, and define the interval $\hat{I}(i,u)$ by extending $I(i,u)$ by $p$ to each direction (that is, $\hat{I}(i,u)$ is of length $3p$).

\item For every hash value $u\in L$, let $B(u)$ denote the box in $\R^{d}$ whose projection on every axis $z_i\in Z$ is $\hat{I}(i,u)$. 

\item Use algorithm \texttt{LDP-AVG} to obtain, for every $u\in L$, an approximation $\hat{y}_{u}$ for the average of 
$
\left\{ x\in S : h(x)=u \text{ and }  x\in B(u) \right\}.
$

\item Use Theorem~\ref{thm:HH} with $\frac{\eps}{4}$ to identify for every $u\in L$ an estimation $\hat{v}(u)\approx v(u)\triangleq |\{x\in S : h(x)=u \text{ and } \|x-\hat{y}_u\|\leq5cr\}|$. Delete from $L$ every element $u\in L$ such that $\hat{v}(u)\leq \frac{t}{4}\cdot n^{-b}$.

\item Output the set of centers $Y=\{ \hat{y}_u : u\in L \}$, the list $L$, and the hash function $h$.

\end{enumerate}
\end{algorithm*}

\begin{observation}
Algorithm \texttt{CentersProcedure} satisfies $\eps$-LDP.
\end{observation}

We now proceed with the utility analysis of algorithm \texttt{CentersProcedure}. We will assume the existence of a family $\HHH$ of $(r,cr,p \smallequality n^{\smallminus b},q \smallequality n^{ \smallminus 2 \smallminus a})$-sensitive hash functions mapping $\R^d$ to a universe $U$, for some constants $a>b$, $r>0$, and $c>1$.

\begin{mylemma}\label{lem:CenterProcedureUtility}
 Let $\beta,\eps,\delta,n,d,\Lambda,r$ be such that $\Lambda/r\leq\poly(n)$ and 
$t\geq O\left( \frac{n^{0.5+b}\cdot \sqrt{d}}{\eps} \log\left(\frac{dn}{\beta\delta}\right) \right)$ and $\beta\leq n^{-a}/28$.
Let $S=(x_1,\dots,x_n)$ be a distributed database where every $x_i$ is a point in the $d$-dimensional ball $\BBB(0,\Lambda)$, and let {\rm \texttt{CentersProcedure}} be executed on $S$ with the family $\HHH$ and with parameters $r,t,\beta,\eps,\delta$. The algorithm outputs a list of hash values $L$, a hash function $h$, and a set $Y$ containing a center $\hat{y}_u$ for every $u\in L$, such that
\begin{enumerate}
	\item The list $L$ and the set $Y$ are size at most $\frac{32 \cdot n^{1+b}}{t}$ each.
	\item With probability at least $1-\beta$, for every $u\in L$ we have $$|\{x\in S: h(x)=u \text{ and } \|x-\hat{y}_u\|\leq5cr\}|\geq\frac{t}{8}\cdot n^{-b}.$$
	\item Let $P\subseteq S$ be a set of $t$ points which can be enclosed in a ball of radius $r$. With probability at least $n^{-a}/4$ there exists $u^*\in L$ such that the ball of radius $3cr$ around $\hat{y}_{u^*}\in Y$ contains at least one point from $P$. 
\end{enumerate}
\end{mylemma}

\begin{remark}
Lemma~\ref{lem:CenterProcedureUtility} can be interpreted as follows. Item~3 states that, with noticeable probability, the set of candidate $Y$ ``captures'' every large enough cluster $P$ of input points that can be enclosed in a ball of radius $r$, in the sense that $Y$ contains a candidate center that is close to this cluster. Item~1 states that the number of candidate centers (i.e., the size of $Y$) is not too big. Item~2 states that every candidate center $y\in Y$ corresponds to a hash value $u\in L$ such that there are ``a lot'' of input points that are hashed to $u$ and are ``close'' to $y$.
\end{remark}

\begin{proof}
Items~1 and~2 of the lemma follow from the fact that in Step~8 we delete from the list $L$ every element $u$ that does not satisfy the condition of item~2. Specifically, for $t\geq O\left(\frac{1}{\eps}n^{0.5+b}\sqrt{\log(\frac{1}{\beta})}\right)$, our estimations in Step~8 are accurate enough such that item~2 holds with probability at least $1-\beta$, in which case the list $L$ is short (a longer list can be trimmed). We now proceed with the analysis of item~3.

First observe that, w.l.o.g., we can assume that the range $U$ of every function in $\HHH$ is of size $|U|\leq n^3$. If this is not the case, then we can simply apply a (pairwise independent) hash function with range $n^3$ onto the output of the locally sensitive hash function. Clearly, this will not decrease the probability of collusion for ``close'' elements (within distance $r$), and moreover, this can increase the probability of collusion for ``non-close'' elements (at distance at least $cr$) by at most $n^{-3}=o(n^{-2-a})=o(q)$.

Now recall that by the properties of the family $\HHH$, for every $x,y\in\R^d$ s.t.\ $\|x-y\|\geq cr$ we have that $\Pr_{h\in\HHH}[h(x)=h(y)]\leq q=n^{-2-a}$. Using the union bound we get 
$$\Pr_{h\in_R\HHH}[h(x)\neq h(y)~\mbox{for all}~x,y\in S~\mbox{s.t.}~\|x-y\|\geq cr] \geq (1-n^{-a}/2).$$ 

Let $P\subseteq S$ denote the guaranteed set of $t$ input points that are contained in a ball of radius $r$, and let $x\in P$ be an arbitrary point in $P$. By linearity of expectation, we have that
$$
\E_{h\in\HHH}\left[ |\{ y\in P : h(y)\neq h(x) \}| \right] \leq t(1-p)=t(1-n^{-b}).
$$
Hence, by Markov's inequality,
$$
\Pr_{h\in\HHH}\left[ |\{ y\in P : h(y)\neq h(x) \}| \geq \frac{t(1-n^{-b})}{1-n^{-a}} \right] \leq 1-n^{-a}.
$$
So,
$$
\Pr_{h\in\HHH}\left[ |\{ y\in P : h(y)= h(x) \}| \geq t\left(1-\frac{1-n^{-b}}{1-n^{-a}}\right) \right] \geq n^{-a}.
$$
Simplifying, for large enough $n$ (specifically, for $n^{a-b}\geq 2 $) we get 
$$
\Pr_{h\in\HHH}\left[ |\{ y\in P : h(y)= h(x) \}| \geq \frac{t}{2}\cdot n^{-b} \right] \geq n^{-a}.
$$

So far we have established that with probability at least $n^{-a}/2$ over the choice of  $h\in\HHH$ in Step~1 the following events occur:
\begin{enumerate}
	\item[$(E_1)$] For every $x,y\in S$ s.t.\ $\|x-y\|\geq cr$ it holds that $h(x)\neq h(y)$; and,
	\item[$(E_2)$] There exists a hash value in $U$, denoted $u^*$, such that $|\{ y\in P : h(y)= u^* \}| \geq \frac{t}{2}\cdot n^{-b}$.
\end{enumerate}

Event $(E_1)$ states that if two points in $S$ are mapped into the same hash value, then these points are close.
Event $(E_2)$ states that there is a ``heavy'' hash value $u^*\in U$, such that ``many'' of the points in $P$ are mapped into $u^*$.
We proceed with the analysis assuming that these two events occur.

On step~2, we identify a list $L$ containing all such ``heavy'' hash values $u\in U$. Assuming that $t\geq O\left(\frac{1}{\eps}\cdot n^{0.5+b}\cdot\sqrt{\log(n/\beta)}\right)$, Theorem~\ref{thm:HH} ensures that with probability at least $1-\beta$ we have that $u^*\in L$. We continue with the analysis assuming that this is the case.

On Step~3 we generate a random orthonormal basis $Z$. By Lemma~\ref{lem:RandomRotation}, with probability at least $(1-\beta)$, for every $x,y\in S$ and for every $z_i\in Z$, we have that the projection of $(x-y)$ onto $z_i$ is of length at most $2\sqrt{\ln(dn/\beta)/d}\cdot\|x-y\|$. In particular, for every hash value $u\in L$ we have that the projection of $S_u\triangleq\{ x\in S \;:\; h(x)= u \}$ onto every axis $z_i\in Z$ fits within an interval of length at most $p=2rc\sqrt{\ln(dn/\beta)/d}$.
Recall that we assume that input point come from $\BBB(0,\Lambda)$. Hence, for every $x,y\in S$ we have $(x-y)\in\BBB(0,2\Lambda)$. Now, as $\III=\{I_1,I_2,\dots\}$ is a partition of $[-2\Lambda,2\Lambda]$ into intervals of length $p$, for every axis $z_i\in Z$ and for every $u\in U$, we have that the projection of $S_u$ onto $z_i$ is contained within 1 or 2 consecutive intervals from $\III$.

On step~4 we partition $S$ into $d$ subsets $S^i\subseteq S$ of size $\frac{n}{d}$. 
By the Hoeffding bound, assuming that $t\geq2\cdot n^{0.5+b}\cdot\sqrt{2d\ln(\frac{2d}{\beta})}$, with probability at least $1-\beta$, for every $i\in[d]$, we have that
$|S^i\cap S_{u^*}|\geq\frac{|S_{u^*}|}{2d}\geq\frac{t\cdot n^{-b}}{4d}$. Recall that the projection of $S_{u^*}$ onto every axis $z_i\in Z$ fits within (at most) 2 consecutive intervals from $\III$. Hence, for every axis $z_i\in Z$, at least 1 interval from $\III$ contains at least half of the points from $S^i\cap S_{u^*}$, i.e., at least $\frac{t\cdot n^{-b}}{8d}$ points. 
Therefore, for $t\geq O\left(\frac{1}{\eps}\cdot n^{0.5+b}\cdot\sqrt{d\cdot\log(\frac{dn}{\beta})}\right)$, Theorem~\ref{thm:HH} ensures that with probability at least $1-\beta$, for every $z_i\in Z$ we have that $I(i,u^*)=\argmax_{i\in\III}\{a_i(I,u^*)\}$ (defined on step~5) contains at least one point from $S_{u^*}$.\footnote{The constraint on $t$ in Theorem~\ref{thm:HH} depends logarithmically on the number of possible bins. In our case, there are $4\Lambda/p\leq\Lambda\sqrt{d}/r\leq \poly(n\sqrt{d})$ possible bins, where the last inequality is because we assumed that $\Lambda/r\leq\poly(n)$.} Hence, the interval $\hat{I}(i,u^*)$ obtained by extending $I(i,u^*)$ by $p$ to each direction, contains (the projection of) {\em all} of the points from $S_{u^*}$ (onto the $i^{\text{th}}$ axis). As a result, the box $B(u^*)$, defined on step~7 as the box whose projection onto every axis $i$ is $\hat{I}(i,u^*)$, contains {\em all} of $S_{u^*}$. We continue with the analysis assuming that this is the case. Observe that the diameter of $B(u^*)$, as well as the diameter of every other box $B(u)$ defined on step 6, is at most $3p\sqrt{d}=6cr \sqrt{\ln(dn/\beta)}=\tilde{O}(cr)$.

On step~7 we use algorithm \texttt{LDP-AVG} to obtain, for every $u\in L$, an estimation $\hat{y}_u$ for the average of 
$\left\{ x\in S : h(x)=u \text{ and }  x\in B(u) \right\}$. Let us denote the true average of every such set as $y_u$. By the properties of \texttt{LDP-AVG} (Claim~\ref{claim:LDP-AVG}), assuming that $t\geq O\left( \frac{1}{\eps}\cdot\sqrt{dn}\cdot \log\left(\frac{dn}{\beta\delta}\right)\right)$, with probability at least $1-\beta$ we have that $\|y_{u^*}-\hat{y}_{u^*}\|_2\leq cr$. We continue with the analysis assuming that this is the case.

Observe that $y_{u^*}$ is the average of (some of) the points in $S_{u^*}$, and that every two points in $S_{u^*}$ are within distance $cr$ from each other. Hence, we get that a ball of radius $2cr$ around $y_{u^*}$ contains all of $S_{u^*}$. 
In particular, as $S_{u^*}$ contains at least some of the points from $P$ (the guaranteed cluster radius $r$ with $t$ input points from $S$), we have that the ball of radius $2cr$ around $y_{u^*}$ contains at least 1 point from $P$, and that the ball of radius $4cr$ around $y_{u^*}$ contains {\em all} of $P$. Therefore, as $\|y_{u^*}-\hat{y}_{u^*}\|_2\leq cr$ we get that a ball of radius $3cr$ around $\hat{y}_{u^*}$ contains at least one point from $P$, and that the ball of radius $5cr$ around $\hat{y}_{u^*}$ contains all of $P$.

Recall that, by Event $(E_2)$, there are at least $\frac{t}{2}\cdot n^{-b}$ input points $x\in P$ such that $h(x)=u^*$. Therefore, in Step~8 we have that $v(u^*)\geq\frac{t}{2}\cdot n^{-b}$. Therefore, with probability at least $1-\beta$ we also have that $\hat{v}(u^*)\geq\frac{t}{2}\cdot n^{-b}$, because when $t\geq O\left(\frac{1}{\eps}n^{0.5+b}\sqrt{\log(\frac{1}{\beta})}\right)$ then the error $|v(u^*)-\hat{v}(u^*)|$ is small compared to $v(u^*)$. This means that $u^*$ is not deleted from the list $L$ in Step~8.

Overall, with probability at least $\frac{n^{-a}}{2}-7\beta$ we have that the output set $Y$ (from Step~9) contains at least one vector $\hat{y}_{u^*}$ s.t.\ the ball of radius $3cr$ around $\hat{y}$ contains at least one point from $P$.
\end{proof}

\subsection{Algorithm \texttt{GoodCenters}}

We are now ready to present algorithm \texttt{GoodCenters}, and to prove Theorem~\ref{thm:GoodCenters} (restated here as Theorem~\ref{thm:GoodCentersFull}). As we mentioned, algorithm \texttt{GoodCenters} is obtained by applying algorithm \texttt{CentersProcedure} multiple times to boost its success probability.

\begin{algorithm*}[!htp]

\caption{\texttt{GoodCenters}}\label{alg:GoodCenters}

{\bf Input:} Radius $r$, target number of points $t$, failure probability $\beta$, privacy parameters $\eps,\delta$.

\smallskip
\noindent {\bf Optional input: }Parameter $t$. Otherwise set $t=O\left( \frac{1}{\eps}\cdot n^{0.5+a+b}\cdot \sqrt{d}\cdot\log(\frac{1}{\beta}) \log\left(\frac{dn}{\beta\delta}\right) \right).$ 

\smallskip
\noindent {\bf Setting: }Each player $j\in[n]$ holds a value $x_j\in \BBB(0,\Lambda)$. Define $S=(x_1,\dots,x_n)$.

\begin{enumerate}[leftmargin=15pt,rightmargin=10pt,itemsep=1pt,topsep=1.5pt]

\item Denote $M=4n^a \ln(\frac{1}{\beta})$ and randomly partition $[n]$ into $M$ subsets $I_1,\dots,I_m\subseteq[n]$. For $m\in[M]$ define $S_m=\{x_i\in S: i\in I_m\}$.

\item For $m\in[M]$ apply algorithm \texttt{CentersProcedure} on $S_m$ with the following parameters: radius $r$, failure probability $\hat{\beta}=\frac{\beta}{7M}$, privacy parameters $\eps,\delta$, and target number of points $\hat{t}=\frac{t}{2M}$. For every $m\in[M]$ denote the outcomes as $Y_m$, $L_m$, and $h_m$.

\item Return the sets of centers $Y_1,\dots,Y_M$, the lists $L_1,\dots,L_M$, the hash functions $h_1,\dots,h_M$, and the partition $I_1,\dots,I_M$.

\end{enumerate}
\end{algorithm*}

\begin{mytheorem}\label{thm:GoodCentersFull}
 Let $\beta,\eps,\delta,n,d,\Lambda,r$ be such that $t\geq O\left( \frac{n^{0.5+a+b}\cdot \sqrt{d}}{\eps}\log(\frac{1}{\beta}) \log\left(\frac{dn}{\beta\delta}\right) \right)$ and such that $\Lambda/r\leq\poly(n)$. 
Let $S=(x_1,\dots,x_n)$ be a distributed database where every $x_i$ is a point in the $d$-dimensional ball $\BBB(0,\Lambda)$, and let {\rm \texttt{GoodCenters}} be executed on $S$ with parameters $r,t,\beta,\eps,\delta$. Denote $M=4n^a \ln(\frac{1}{\beta})$. 
The algorithm outputs a partition $I_1,\dots,I_M\subseteq[n]$, lists $L_1,\dots,L_M$, hash functions $h_1,\dots,h_M$, and sets of centers $Y_1,\dots,Y_M$, where for every $m\in[M]$ and $u\in L_m$, the set $Y_m$ contains a center $\hat{y}_{m,u}$. The following holds.
\begin{enumerate}
	\item With probability at least $1-\beta$, for every $m\in[M]$ and $u\in L_m$ we have $|\{i\in I_m: h_m(x_i)=u \text{ and } \|x_i-\hat{y}_{m,u}\|\leq5cr\}|\geq\frac{t}{16M}\cdot n^{-b}$.
	\item Denote $Y=\bigcup_{i\in[M]}Y_m$. Then $|Y|\leq\frac{512 \cdot n^{1+a+b}}{t}\ln(\frac{1}{\beta})$.
	\item Let $P\subseteq S$ be a set of $t$ points which can be enclosed in a ball of radius $r$. With probability at least $1-\beta$ there exists $\hat{y}\in Y$ such that the ball of radius $5cr$ around $\hat{y}$ contains all of $P$. 
\end{enumerate}
\end{mytheorem}

\begin{proof}
Items~1 and~2 of the lemma follow directly from the properties of algorithm \texttt{CentersProcedure}. We now proceed with the analysis of item~3. 
Let $P\subseteq S$ be s.t.\ $|P|=t$ and $\diam(P)\leq r$, and consider the following good event, which happens with probability at least $1-\beta$ by the Chernoff bound (assuming that $t\geq 24 M \ln(\frac{eM}{\beta})$).

\begin{center}
\noindent\fboxother{
\parbox{.9\columnwidth}{
{\bf Event ${\boldsymbol{E_1}}$ (over partitioning ${\boldsymbol{S}}$ into ${\boldsymbol{S_1,\dots,S_M}}$): }
\begin{enumerate}[itemsep=0pt,topsep=4pt]
	\item[(a)] For every $m\in[M]$ we have $|P\cap S_m| \geq\frac{t}{2M}$.
	\item[(b)] For every $m\in[M]$ we have $\frac{n}{2M}\leq|I_m|\leq\frac{2n}{M}$.
\end{enumerate}

}}
\end{center}

We proceed with the analysis assuming that Event $E_1$ occurred. 
Let us say that the $m$th execution of \texttt{CentersProcedure} {\em succeeds} if $\exists y_m\in Y_m$ such that the ball of radius $3cr$ around $y_m$ contains at least one point from $P\cap S_m$. Recall that we assume that $t\geq O\left(\frac{n^{0.5+a+b}\cdot\sqrt{d}}{\eps}\log(\frac{1}{\beta})\log(\frac{dn}{\beta\delta})\right)$. Hence, by the properties of algorithm \texttt{CentersProcedure}, every single execution succeeds with probability at least $n^{-a}/4$. As the different executions of \texttt{CentersProcedure} are independent, when $M\geq4n^a \ln(\frac{1}{\beta})$, the probability that at least one execution succeeds is at least $1-\beta$. In this case, there is a point $y\in Y=Y_1\cup\dots\cup Y_m$ such that the ball of radius $3cr$ around it contains at least one point from $P$, and hence, the ball of radius $5cr$ around $y$ contains {\em all} of $P$.
\end{proof}

\end{document}